\documentclass{article}
\RequirePackage{zaiweiandci}

\title{\LARGE \bf  Finite-Sample Analysis of Stochastic Approximation Using Smooth Convex Envelopes}

\author{
{\normalsize Zaiwei Chen}\thanks{Georgia Institute of Technology, {\color{blue}zchen458@gatech.edu}}
\and
{\normalsize Siva Theja Maguluri}\thanks{Georgia Institute of Technology, {\color{blue}siva.theja@gatech.edu }}
\and
{\normalsize Sanjay Shakkottai\thanks{The University of Texas at Austin, {\color{blue}sanjay.shakkottai@utexas.edu}}}
\and
{\normalsize Karthikeyan Shanmugam}\thanks{IBM Research NY, {\color{blue}Karthikeyan.Shanmugam2@ibm.com}}
}

\date{}
\begin{document}
\maketitle

\begin{abstract}
    Stochastic Approximation (SA) is a popular approach for solving fixed-point equations where the information is corrupted by noise. In this paper, we consider an SA involving a contraction mapping with respect to an arbitrary norm, and show its finite-sample error bounds while using different stepsizes. The idea is to construct a smooth Lyapunov function using the generalized Moreau envelope, and show that the iterates of SA have negative drift with respect to that Lyapunov function.  Our result is applicable in Reinforcement Learning (RL). In particular, we use it to establish the first-known convergence rate of the V-trace algorithm for off-policy TD-learning \cite{espeholt2018impala}. Moreover, we also use it to study TD-learning in the on-policy setting, and recover the existing state-of-the-art results for $Q$-learning. Importantly, our construction results in only a logarithmic dependence of the convergence bound on the size of the state-space.
\end{abstract}

\section{Introduction}\label{sec:intro}
Reinforcement Learning (RL) captures an important facet of machine learning going beyond prediction and regression: sequential decision making, and has had great impact in various problems of practical interest \cite{silver2016mastering,mirowski2018learning,shalev2016safe}. At the heart of RL is the problem of iteratively solving the Bellman's equation using noisy samples, i.e. solving a fixed-point equation of the form $\mathcal{H}(x)=x$. Here, $\mathcal{H}$ is a contractive operator with respect to a suitable norm, where we only have access to samples from noisy versions of the operator. Such fixed-point equations, more broadly, are solved through the framework of Stochastic Approximation (SA) algorithms \cite{robbins1951stochastic}, with several RL algorithms such as $Q$-learning and TD-learning being examples there-of. This paper focuses on understanding the evolution of such a noisy fixed-point iteration through the lens of SA, and providing finite-sample convergence results.

More formally, the SA algorithm for solving the fixed-point equation $\mathcal{H}(x)=x$ is of the form $x_{k+1}=x_k+\epsilon_k\left(\mathcal{H}(x_k)-x_k+w_k\right)$, where $\{\epsilon_k\}$ is the stepsize sequence, and $\{w_k\}$ is the noise sequence. To derive finite-sample bounds, three conditions are pertinent: (a) The norm in which the operator $\mathcal{H}$ contracts, (b) The mean zero noise when conditioned on the past, and (c) The nature of the bound on the conditional second moment of the noise.

In prior literature, if the conditional second moment of the noise $\{w_k\}$ is uniformly bounded by a constant, then the norm with respect to which $\mathcal{H}$ being a contraction becomes irrelevant, and it is possible to derive finite-sample convergence guarantees \cite{beck2012error,beck2013improved,even2003learning,dvoretzky1956}. When the second moment of the noise is not uniformly bounded, then finite-sample bounds can  be derived in the case where the norm for contraction of $\mathcal{H}$ is the Euclidean norm \cite{bertsekas1996neuro,bottou2018optimization}.
However, in many RL problems, the contraction of $\mathcal{H}$ occurs with respect to a different norm (e.g. the $\ell_{\infty}$-norm \cite{watkins1992q} or a weighted variant \cite{tsitsiklis1994asynchronous}). Further, conditioned on the past, the second moment of the norm of the noise scales affinely  with the current iterate (again w.r.t. an arbitrary norm), and in general, no uniform bound exists.

An important practical application of this setting with $\ell_\infty$-norm contraction and unbounded noise is the well-known V-trace algorithm for solving the policy evaluation problem using off-policy TD-learning \cite{sutton2018reinforcement}. Its variants form the basis of today's distributed RL platforms like IMPALA \cite{espeholt2018impala} and TorchBeast \cite{kuttler2019torchbeast} for multi-agent training. It has been used at scale in the recent Deepmind City Navigation Project “Street Learn”  \cite{mirowski2018learning}. Therefore, deriving finite-sample convergence results for SA under contraction of $\mathcal{H}$ with respect to general norms, and handling unbounded noise are of fundamental interest. In this paper, we answer the following general question in the affirmative:

\textit{Can we provide finite-sample convergence guarantees for the SA algorithm when the norm of contraction of $\mathcal{H}$ is arbitrary, and the second moment of the noise conditioned on the past scales affinely with respect to the squared-norm of the current iterate?}

To the best of our knowledge, except under special conditions on the norm for contraction of $\mathcal{H}$ and/or strong assumptions on the noise, such finite-sample error bounds have not been established. For a summary of related works, please see the following table.

\begin{table}[h]
\begin{center}
    \caption{Comparison to existing bounds.\label{table:1}}
	{\begin{tabular}{|c|c|c|c|c|c|}
			\hline
			\multirow{2}{*}{\centering Topic}
			& \multirow{2}{5 em}{\centering Contraction Norm}& \multirow{2}{5 em}{\centering Noise Assumption} &\multirow{2}{*}{\centering Stepsize}  &\multirow{2}{*}{\centering Rate}  & \multirow{2}{*}{\centering $d$-dependence}  \\
			&&&&&\\	
			\hline
			 $Q$-learning \cite{beck2012error,beck2013improved}&  $\|\cdot\|_\infty$ & Bounded & Constant& Geometric &  $d^2$\\
			\hline
			 $Q$-learning \cite{wainwright2019stochastic}&  $\|\cdot\|_\infty$ & Bounded & $1/(1+(1-\gamma)k)$ & $O(1/k)$ &  $\log (d)$\\
			\hline
			 $Q$-learning \cite{wainwright2019stochastic}&  $\|\cdot\|_\infty$ & Bounded & $1/k^\xi$ & $O(1/k^\xi)$ &  $\log (d)$\\
			\hline
			 SGD \cite{bottou2018optimization}& $\|\cdot\|_2$ &  Affine &  Constant &  Geometric &  Independent\\
			\hline
			 SGD \cite{bottou2018optimization} & $\|\cdot\|_2$ &  Affine &   $\beta/(\gamma+k)$ & $O(1/k)$ &  Independent\\
			\hline
			 $Q$-learning & \multirow{2}{*}{$\|\cdot\|_\infty$} & \multirow{2}{*}{Bounded} & \multirow{2}{*}{Constant} & \multirow{2}{*}{Geometric}& \multirow{2}{*}{$\log(d)$} \\\relax
			[This Work] & & & & & \\
			\hline
			 V-trace & \multirow{2}{*}{$\|\cdot\|_\infty$} & \multirow{2}{*}{Affine} & \multirow{2}{*}{$\epsilon/(k+K)$} & \multirow{2}{*}{$O(1/k)$}& \multirow{2}{*}{$\log(d)$} \\\relax
			[This Work] & & & & & \\
			\hline
			 Contractive SA & Arbitrary & \multirow{2}{*}{Affine} & Constant \& & Corollary \ref{co:constant_step_size} & \multirow{2}{*}{$\log(d)$ ($\|\cdot\|_\infty$)} \\\relax
			[This Work] & norm & & Diminishing & Corollary \ref{co:diminishing_step_size} & \\
			\hline
	\end{tabular}}
\end{center}

	{The $d$-dependence refers to the dependence on the dimension $d$ of the iterate $x_k$. To clarify, in the corresponding $d$-dependence of this work for general contractive SA, we write $\log (d)$ $(\|\cdot\|_\infty)$ to indicate that the dimension dependence is $\log(d)$ when the norm of contraction is the $\ell_\infty$-norm.
	}
\end{table}

The \textbf{main contributions} of this paper are as follows.

\begin{itemize}
	\item \textbf{Finite-Sample Convergence Guarantees for SA.} We present a novel approach for deriving finite-sample error bounds of the SA algorithm under a general norm contraction. The key idea is to study the drift of a carefully constructed potential/Lyapunov function. We obtain such a  potential function by smoothing the norm-squared function using a generalized Moreau envelope. We then study the error bound under either constant or diminishing stepsizes. Specifically, we show that the iterates converge to a ball with radius proportional to the stepsize when using constant stepsize, and converge with rate roughly $O(1/k)$ when using diminishing stepsizes of proper decay rate.	
	\item \textbf{Performance of the V-trace Algorithm.} To demonstrate the effectiveness of the theoretical result in an entirely novel setting in RL, we consider the V-trace algorithm (in the synchronous setting) for solving the policy evaluation problem using off-policy sampling \cite{espeholt2018impala}. Interestingly in this case, it is not clear if the iterates of the V-trace algorithm are uniformly bounded by a   constant (e.g. as in $Q$-learning \cite{gosavi2006boundedness}). Therefore, existing techniques are not applicable. Using our approach, we establish the first known finite-sample error bounds, and show that the convergence rate is logarithmic in the state-space dimension. In our result, the logarithmic dimension dependence relies on the general form of the Moreau envelope obtained by  the infimal convolution with a suitable smooth squared-norm. The freedom in selecting such norm allows us to obtain the logarithmic dependence. 
	\item \textbf{Performance of the TD$(n)$ Algorithm.} Our results can also be used to study the $n$-step TD-learning algorithm (with synchronous update) in the on-policy setting. Specifically, by establishing convergence bounds, we are able to analyze the performance of the TD$(n)$ algorithm for different choices of $n$. The conclusion is supported by a numerical example.
	\item \textbf{Performance of the $Q$-Learning Algorithm.} Through the smooth Lyapunov approach, our results recover existing state-of-the-art finite-sample bounds for $Q$-learning that show only a logarithmic dependence on the size of the state-action space \cite{wainwright2019stochastic}. Specifically, we match the results in  \citep{wainwright2019stochastic} in a diminishing stepsize regime, and improve over \citep{beck2012error,beck2013improved} in a constant stepsize regime.
\end{itemize}

\subsection{Summary of Our Techniques}\label{sec:Summary}
We now give a more detailed description of the techniques we used. To provide intuition, assume for now that the norm $\|\cdot\|_c$ with respect to which $\mathcal{H}$ being a contraction is the $\ell_p$-norm for $p\in [2,\infty)$, i.e., $\|\mathcal{H}(x)-\mathcal{H}(y)\|_p\leq \gamma\|x-y\|_p$ for all $x,y\in\mathbb{R}^d$, where $\gamma\in (0,1)$ is the contraction factor. Denote the fixed-point of $\mathcal{H}$ by $x^*$. Consider the Ordinary Differential Equation (ODE) associated with this SA: $\dot{x}(t) = \mathcal{H}(x(t))-x(t)$. It is shown in \cite{borkar2009stochastic} (Chapter $10$) that the function $W(x)=\|x-x^*\|_p$ satisfies $\frac{d}{dt}W(x(t))\leq -\alpha W(x(t))$ for some $\alpha>0$, which implies the solution $x(t)$ of the ODE converges to its equilibrium point $x^*$ geometrically fast. The term $\alpha$ corresponds to a \textit{negative drift}.

In order to obtain finite-sample bounds, in this paper we study the SA directly, and not the ODE. Then, the Lyapunov function $W(x)$ cannot be directly used to analyze the SA algorithm due to the discretization error and stochastic error. However, suppose we can find a function $M(x)$ that gives negative drift, and in addition: (a) $M(x)$ is $L$ -- smooth w.r.t. some norm $\|\cdot\|_s$ \cite{beck2017first}, (b) the noise $\{w_k\}$ is zero mean conditioned on the past, and (c) the conditional second moment of $\|w_k\|_e$ (where $\|\cdot\|_e$ is again some arbitrary norm) can be bounded affinely by the current iterate $\|x_k\|_e^2$. Then, we have a handle to deal with the discretization error and error caused by the noise to obtain:
\begin{align}\label{eq:contraction}
	\mathbb{E}[M(x_{k+1} - x^*)] \leq (1- O(\epsilon_k) + o(\epsilon_k) ) \mathbb{E}[M(x_k - x^*)] + o( \epsilon_k),
\end{align}
which implies a contraction in $\mathbb{E}[M(x_{k+1}-x^*)]$. Therefore, a finite-sample error bound can be obtained by recursively applying the previous inequality. The key point is that \textit{$M(x)$'s smoothness and its negative drift with respect to the ODE produces a contraction $(1- O(\epsilon_k) + o (\epsilon_k))$ for $\{x_k\}$}. Based on the above analysis, we see that the Lyapunov function for the SA in the case of $\ell_p$-norm contraction should be $M(x)=\frac{1}{2}\|x-x^*\|_p^2$, which is known to be smooth \cite{beck2017first}.

However, in the case where the contraction norm $\|\cdot\|_c$ is arbitrary, since the function $f(x)=\frac{1}{2}\|x-x^*\|_c^2$ is not necessarily smooth, the key difficulty is to construct a smooth Lyapunov function. An important special case is when $\|\cdot\|_c = \|\cdot\|_\infty$, which is applicable to many RL algorithms. We provide a solution to this where we construct a smoothed convex envelope $M(x)$ called the \textit{Generalized Moreau Envelope} that is smooth w.r.t. some norm $\|\cdot\|_s$, and it is a tight approximation to $f(x)$, i.e. $(1+a)M(x) \leq f(x) \leq (1+b)M (x)$ for some small enough constants $a,b$. Further, it is a Lyapunov function for the ODE with a negative drift. This essentially lets us prove a convergence result akin to the case when $f(x)$ is smooth.

\subsection{Related Work}
Due to the popularity of the SA algorithm (and its variant Stochastic Gradient Descent (SGD) in optimization \cite{nemirovski2009robust,lan2020first}), it has been studied extensively in the literature. Specifically, suppose that noise $\{w_k\}$ is a martingale difference sequence with some mild conditions on its variance, and the stepsize decays to zero at an appropriate rate. Then, almost sure convergence of the sequence $\{x_k\}$ has been established in \cite{tsitsiklis1994asynchronous,jaakkola1994convergence} using a supermartingale convergence approach, and in \cite{borkar2000ode,borkar2009stochastic} using an ODE approach. Further, when the iterates are uniformly bounded by an absolute constant (with probability $1$), or that the operator $\mathcal{H}$ is contractive with respect to the Euclidean norm, {\em convergence rates and finite-sample bounds} can be derived using the decomposition methods \cite{tsitsiklis1994asynchronous} or Lyapunov techniques \cite{bertsekas1996neuro}. In particular, the decomposition technique has been used for $Q$-learning in \cite{beck2012error,beck2013improved,wainwright2019stochastic} to derive finite-sample convergence bounds, using the fact the iterates of $Q$-learning are uniformly bounded by a constant \cite{gosavi2006boundedness}. Concentration results for $Q$-learning were also derived in \cite{qu2020finite,li2020sample}. As for TD-learning and $Q$-learning with linear function approximation, finite-sample guarantees were shown in \cite{dalal2018finite,bhandari2018finite,srikant2019finite,chen2019finitesample} for a single-agent problem, and in \cite{doan2019finite} for a multi-agent problem.  Concentration results for SA algorithm when starting near an attractor of the underlying ODE were derived in \cite{borkar2000number,thoppe2019concentration}. Variations of temporal difference methods (such as gradient TD, least squares TD) have been studied and their convergence has been analyzed in some cases in \cite{yu2009convergence,sutton2009convergent,sutton2009fast}.

Moreau envelopes are popular tools for non-smooth optimization \cite{moreauoriginalfrench}, where the proximal operator is used to develop algorithms to work with non-smooth parts of the objective \cite{beck2012smoothing}. Moreau envelopes have been used as potential functions to analyze convergence rate of subgradient methods to first order stationary points for non-smooth and non-convex stochastic optimization problems in \cite{davis2019stochastic}. They use the Moreau envelope defined with respect to the Euclidean norm, and use this to show convergence by bounding a measure on first order stationarity with the gradient of the Moreau envelope. In contrast, our interest is in understanding contraction with arbitrary norms -- this requires us to use a generalized Moreau envelope obtained by infimal convolution with a general smooth function, and show that its a smooth Lyapunov function with respect to the underlying ODE. The flexibility in the selection of this smooth function in our infimal convolution plays a crucial role in improving the dependence on the state-space dimension to logarithmic factors for our applications.

The rest of this paper is organized as follows. In Section \ref{sec:sa} we introduce the SA algorithm under contraction assumption and derive its finite-sample convergence bounds for using different stepsizes. In Section \ref{sec:applications} we apply our result to the context of Reinforcement Learning and  study the convergence rates of the V-trace algorithm \cite{espeholt2018impala}, the on-policy TD$(n)$ algorithm \cite{sutton2018reinforcement}, and the $Q$-learning algorithm \cite{watkins1992q}.

\section{Stochastic Approximation under A Contraction Operator}\label{sec:sa}	
\subsection{Problem Setting}
\label{subsec:noisySA}	
Let $\mathcal{H}:\mathbb{R}^d\mapsto\mathbb{R}^d$ be a nonlinear mapping. We are interested in solving for $x^*\in\mathbb{R}^d$ in the fixed-point equation $\mathcal{H}(x)=x$. Suppose we have access to the mapping $\mathcal{H}$ only through a noisy oracle which for any $x$ returns $\mathcal{H}(x)+w$. Here $w$ represents the noise in estimating $\mathcal{H}(x)$, and hence $w$ might depend on $x$. In this setting, the following stochastic iterative algorithm is proposed to estimate $x^*$:
\begin{align}\label{sa:algorithm}
	x_{k+1}=x_k+\epsilon_k\left(\mathcal{H}(x_k)-x_k+w_k\right),
\end{align}
where $\{\epsilon_k\}$ is the stepsize sequence \cite{bertsekas1996neuro}. We next state our main assumptions in studying this SA. Let $\mathcal{F}_k=\{x_0,w_0,...,x_{k-1},w_{k-1},x_k\}$, and let $\|\cdot\|_c$ and $\|\cdot\|_e$ be two arbitrary norms in $\mathbb{R}^d$.
\begin{assumption}\label{as:contraction}
	The function $\mathcal{H}$ is a contraction mapping w.r.t. norm $\|\cdot\|_{c}$, i.e., there exists $\gamma\in (0,1)$ such that $\|\mathcal{H}(x)-\mathcal{H}(y)\|_{c}\leq \gamma \|x-y\|_{c}$ for all $x,y\in\mathbb{R}^d$.
\end{assumption}
\begin{remark}
	By Banach fixed-point theorem \cite{debnath2005introduction}, Assumption \ref{as:contraction} implies the existence and uniqueness of the fixed-point $x^*$ of the operator $\mathcal{H}$. In fact, the contraction assumption can be relaxed to pseudo-contraction assumption, i.e., $\|\mathcal{H}(x)-x^*\|_{c}\leq \gamma \|x-x^*\|_{c}$ for all $x\in\mathbb{R}^d$ \cite{bertsekas1996neuro}, which automatically implies that $x^*$ is the unique fixed-point of $\mathcal{H}$.
\end{remark}

\begin{assumption}\label{as:noise}
	The noise sequence $\{w_k\}$ satisfies for all $k\geq 0$: (a) $\mathbb{E}[w_k\mid\mathcal{F}_k]=0$, and (b) $\mathbb{E}[\|w_k\|_e^2\mid\mathcal{F}_k]\leq A+B\|x_k\|_e^2$ for some constants $A,B>0$.
\end{assumption}
\begin{assumption}\label{as:step_size}
	The stepsize sequence $\{\epsilon_k\}$ is positive and non-increasing.
\end{assumption}	

The asymptotic convergence of $x_k$ under similar assumptions has been established in the literature. In particular, an approach based on studying the ODE $\dot{x}(t)=\mathcal{H}(x(t))-x(t)$ was used in \cite{borkar2000ode,borkar2009stochastic}, where it was shown that $x_k$ converges to $x^*$ almost surely under some stability assumptions of the ODE. The focus of this paper is to establish the finite-sample mean square error bounds for SA algorithm (\ref{sa:algorithm}). We do this by studying the drift of a smooth potential/Lyapunov function \cite{srikant2019finite,chen2019finitesample}. While we do not explicitly use the ODE approach, the potential function we are going to contruct in the next subsection is inspired by the Lyapunov function used to study the ODE.

\subsection{The Generalized Moreau Envelope as A Smooth Lyapunov Function}\label{subsec:smooth-approx}

Recall from Eq. (\ref{eq:contraction}) that with respect to the iterates $\{x_k\}$ of the SA , an ideal Lyapunov function $M(x)$ acts as a potential function that contracts. In this subsection, we first construct a Lyapunov function that is smooth through the generalized Moreau envelope. Smoothness and an approximation property of the Lyapunov function we specify here are used in the next subsection to show the contraction property we desire.

To construct such a Lyapunov function, the following definitions are needed. In this paper, $\langle x,y\rangle=x^\top y$ represents the standard dot product, while the norm $\|\cdot\|$ in the following definition can be any arbitrary norm instead of just being the Euclidean norm $\|x\|_2=\langle x,x\rangle^{1/2}$.
\begin{definition}\label{df:smooth}
	Let $g:\mathbb{R}^d\rightarrow\mathbb{R}$ be a convex, differentiable function. Then $g$ is said to be $L$ -- smooth w.r.t. the norm $\|\cdot\|$ if and only if 
	\begin{align*}
		g(y)\leq g(x)+\langle \nabla g(x),y-x\rangle+\frac{L}{2}\|x-y\|^2,\quad \forall\;x,y\in\mathbb{R}^d.
	\end{align*}
\end{definition}

\begin{definition}[Generalized Moreau Envelope \cite{guzman2015lower,beck2012smoothing}]\label{df:moreau}
	Let $h_1:\mathbb{R}^d\mapsto \mathbb{R}$ be a closed and convex function, and let $h_2:\mathbb{R}^d\mapsto\mathbb{R}$ be a convex and $L$ -- smooth function. For any $\mu>0$, the generalized Moreau envelope of $h_1$ w.r.t. $h_2$ is defined by 
	\begin{align*}
		M_{h_1}^{\mu,h_2}(x)=\inf_{u\in\mathbb{R}^d}\left\{h_1(u)+\frac{1}{\mu}h_2(x-u)\right\}.
	\end{align*}
\end{definition}

As an aside, we note that for any two functions $h_1,h_2:\mathbb{R}^d\mapsto\mathbb{R}$, the function defined by $(h_1\square h_2)(x):=\inf_{u\in\mathbb{R}^d}\{h_1(u)+h_2(x-u)\}$ is called the infimal convolution of $h_1$ and $h_2$. Therefore, the generalized Moreau envelope in Definition \ref{df:moreau} can be written as $M_{h_1}^{\mu,h_2}(x)=(h_1\square \frac{h_2}{\mu})(x)$.

Let $f(x)=\frac{1}{2}\|x\|_c^2$, where $\|\cdot\|_c$ is given in Assumption \ref{as:contraction}. Let $\|\cdot\|_{s}$ be an arbitrary norm in $\mathbb{R}^d$ such that $g(x):=\frac{1}{2}\|x\|_s^2$ is $L$ -- smooth w.r.t. the same norm $\|\cdot\|_s$ in its definition. For example, $\|\cdot\|_s$ can be the $\ell_p$-norm for any $p\in [2,\infty)$ (Example 5.11 \cite{beck2017first}). Due to the norm equivalence in $\mathbb{R}^d$ \cite{lax1997linear}, there exist $\ell_{cs},\ell_{es}\in (0,1]$ and $u_{cs},u_{es}\in [1,\infty)$ that depend only on the dimenson $d$ and universal constants, such that $\ell_{cs}\|\cdot\|_c\leq \|\cdot\|_s\leq u_{cs}\|\cdot\|_c$ and $\ell_{es}\|\cdot\|_e\leq \|\cdot\|_s\leq u_{es}\|\cdot\|_e$.

\textbf{Intuition:} With a suitable choice of $\mu$, we will use the Moreau envelope of $f(x)$ with respect to $g(x)$, i.e., $M_f^{\mu,g}(x)=\min_{u\in\mathbb{R}^d}\{f(u)+g(x-u)/\mu\}$ as our Lyapunov function to analyze the behavior of Algorithm (\ref{sa:algorithm}), where the attainment of the minimum can be justified by Theorem 2.14 of \cite{beck2017first}. Intuitively, note that the contraction of $\mathcal{H}$ is w.r.t. $\|\cdot\|_c$, hence the Lyapunov function should be defined in terms of $f(x)$. However, since the function $f(x)$ itself may not be well-behaved (e.g. smooth), we use $g(x)$ as a smoothing function to modify $f(x)$ to obtain $M_f^{\mu,g}(x)$. In order for $M_f^{\mu,g}(x)$ to be a valid Lyapunov function, we need to establish the following two properties: (a) $M_f^{\mu,g}(x)$ should be a smooth function for us to handle the discretization error and the stochastic error in Algorithm (\ref{sa:algorithm}), and (b) $M_f^{\mu,g}(x)$ should be close to the original function $f(x)$ so that we can use the contraction of $\mathcal{H}$ w.r.t. $\|\cdot\|_c$ to establish the overall contraction of the iterates $\{x_k\}$ w.r.t. $M_f^{\mu,g}(x)$. The following Lemma provides us the desired properties. See Appendix \ref{pf:le_properties_moreau} for its proof.

\begin{lemma}[Smoothness and Approximation of the Envelope] \label{le:properties_moreau}
	The generalized Moreau envelope $M_f^{\mu,g}(x)$ has the following properties: 
	\begin{enumerate}
		\item $M_f^{\mu,g}$ is convex and $L/\mu$ -- smooth w.r.t. norm $\|\cdot\|_s$;
		\item It holds for all $x\in\mathbb{R}^d$ that $(1+\mu/u_{cs}^2)M_f^{\mu,g}(x)\leq f(x)\leq (1+\mu/\ell_{cs}^2)M_f^{\mu,g}(x)$;
		\item There exists a norm, denoted by $\|\cdot\|_M$, such that $M_f^{\mu,g}(x)=\frac{1}{2}\|x\|_M^2$ for all $x\in\mathbb{R}^d$.
	\end{enumerate}
\end{lemma}

Lemma \ref{le:properties_moreau} part $1$ is restated from \cite{beck2017first}, and we include it here for completeness. This, together with Lemma \ref{le:properties_moreau} part $2$ implies that $M_f^{\mu,g}(x)$ is a smooth approximation of the function $f(x)$. Lemma \ref{le:properties_moreau} part $3$ indicates that $M_f^{\mu,g}(x)$ is in fact a scaled squared norm. Note that there are two tunable parameters in the contruction of $M_f^{\mu,g}$: the choice of the smoothing norm $\|\cdot\|_s$ and the parameter $\mu$. Eventually, we need to carefully choose $\|\cdot\|_s$ and $\mu$ to optimize the convergence bounds of Algorithm (\ref{sa:algorithm}), and this procedure will be elaborated in detail in Section \ref{subsec:corollaries}.

\subsection{Recursive Contractive Bounds for the Generalized Moreau Envelope}
\label{subsec:recurse-contract}

In this subsection, using smoothness of $M_f^{\mu,g}(x)$ and the fact that $M_f^{\mu,g}(x)$ is an approximation to the function $f(x)$ (both properties derived in Lemma \ref{le:properties_moreau}), we  derive in the following proposition the desired one-step contraction inequality of $M_f^{\mu,g}(x_k-x^*)$, whose proof is presented in Appendix \ref{pf:le_recursion}. To present the coming proposition, we need to define a few more constants. Let
\begin{align*}
	\alpha_1=\frac{1+\mu/\ell_{cs}^2}{1+\mu/u_{cs}^2},\;\alpha_2=1-\gamma\alpha_1^{1/2},\;\alpha_3=\frac{4u_{cs}^2u_{es}^2(B+2)L(\ell_{cs}^2+\mu)}{\mu\ell_{cs}^2\ell_{es}^2},\text{ and }\alpha_4=\frac{\alpha_3}{2(B+2)}.
\end{align*}
The constant $\mu$ should be chosen such that $\alpha_2>0$, which is always possible since $\lim_{\mu\rightarrow 0}\alpha_1=1$ and $\gamma\in (0,1)$.

\begin{proposition}\label{prop:recursion}
	The following inequality holds for all $k\geq 0$:
	\begin{align}\label{eq:recursion}
		\mathbb{E}[M_f^{\mu,g}(x_{k+1}-x^*)\mid\mathcal{F}_k]
		\leq\;&(1-2\alpha_2\epsilon_k+\alpha_3\epsilon_k^2)M_f^{\mu,g}(x_k-x^*)+\frac{\alpha_4(A+2B\|x^*\|_c^2)}{2(1+\mu/\ell_{cs}^2)}\epsilon_k^2.
	\end{align}
\end{proposition}

From Eq. (\ref{eq:recursion}), we see that $\alpha_2$ represents the real contraction effect of the algorithm, and it should be positive, which leads to our feasible range of $\mu$. On the RHS of Eq. (\ref{eq:recursion}), the first term represents the {\em overall contraction} property that results from a combination of the contraction in the drift term that counteracts an expansion resulting from the discretization error and the noise second moment that scales affinely in $\|x_k\|_e^2.$ The second term is a consequence of discretization and the noise $\{w_k\}$. This is the key step in our proof compared to \cite{beck2012error,beck2013improved,qu2020finite} as {\em we do not decompose the analysis into one for the contraction terms and another for the noise terms.}

\subsection{Putting together: Finite-Sample Convergence Bounds}

From Proposition \ref{prop:recursion}, the finite-sample error bound of Algorithm (\ref{sa:algorithm}) can be established by repeatedly using Eq. (\ref{eq:recursion}), which leads to our main result in the following. See Appendix \ref{pf:thm:sa-finite-time-bound} for its proof.
\begin{theorem}\label{thm:sa-finite-time-bound}
	Consider iterates $\{x_k\}$ of Algorithm (\ref{sa:algorithm}). Suppose Assumptions \ref{as:contraction} -- \ref{as:step_size} are satisfied and $\epsilon_0\leq \alpha_2/\alpha_3$. Then we have for all $k\geq 0$:
	\begin{align}\label{eq:finite_bound}
		\mathbb{E}\left[\|x_k-x^*\|_c^2\right]\leq\alpha_1\|x_0-x^*\|_c^2\prod_{j=0}^{k-1}(1-\alpha_2\epsilon_j)+\alpha_4(A+2B\|x^*\|_c^2)\sum_{i=0}^{k-1}\epsilon_i^2\prod_{j=i+1}^{k-1}(1-\alpha_2\epsilon_j).
	\end{align}
\end{theorem}

In Eq. (\ref{eq:finite_bound}), the first term represents how fast the initial condition is forgotten, hence it is proportional to the error in our initial guess $\|x_0-x^*\|_c^2$. The second term represents the impact of the variance in our estimate. Note that the condition $\epsilon_0\leq \alpha_2/\alpha_3$ is made only for ease of exposition. If it is not true, as long as $\lim_{k\rightarrow\infty}\epsilon_k<\alpha_2/\alpha_3$, we can let $k':=\min\{k\geq 0:\epsilon_k\leq \alpha_2/\alpha_3\}$, and then recursively apply Eq. (\ref{eq:recursion}) starting from the $k'$-th iteration. For $k=0,1,...,k'$, it can be easily shown using Eq. (\ref{eq:recursion}) that $\mathbb{E}[\|x_k-x^*\|_c^2]$ is bounded.

{\em Theorem \ref{thm:sa-finite-time-bound} is our key contribution in that it holds in the case when: (a) the contraction of $\mathcal{H}$ can be w.r.t. any general norms, and  (b) the conditional second moment of the noise is not bounded by a constant but in fact scales affinely in the current iterates.} As far as we are aware, Theorem \ref{thm:sa-finite-time-bound} establishes the first-known finite-sample convergence bounds in these general settings.

\subsection{Results with Various Stepsize Regimes and Dimension Dependence}\label{subsec:corollaries}

Upon obtaining a finite-sample error bound in its general form in Theorem \ref{thm:sa-finite-time-bound}, we next consider two common choices of stepsizes, and see what does Eq. (\ref{eq:finite_bound}) give us. We first consider using constant stepsize (i.e., $\epsilon_k\equiv\epsilon\leq \alpha_2/\alpha_3$) in the following result, whose proof is presented in Appendix \ref{pf:co_constant_step_size}.

\begin{corollary}\label{co:constant_step_size}
	Suppose $\epsilon_k\equiv \epsilon\leq \frac{\alpha_2}{\alpha_3}$. Then we have for all $k\geq 0$:
	\begin{align*}
		\mathbb{E}\left[\|x_k-x^*\|_c^2\right]\leq \alpha_1\|x_0-x^*\|_c^2(1-\alpha_2\epsilon)^k+(A+2B\|x^*\|_c^2)\frac{\alpha_4\epsilon}{\alpha_2}
	\end{align*}
\end{corollary}

From Corollary \ref{co:constant_step_size}, we see that in expectation, the iterates converge exponentially fast in the mean square sense, to a ball with radius proportional to the stepsize $\epsilon$, centered at the fixed-point $x^*$. With smaller stepsize, at the end the estimate $x_k$ of $x^*$ is more accurate, but the rate of convergence is slower since the geometric ratio $(1-\alpha_2\epsilon)$ is larger.

We next consider using diminishing stepsizes of the form $\epsilon_k=\epsilon/(k+K)^\xi$, where $\epsilon>0$, $\xi\in (0,1]$, $K=\max(1,\epsilon\alpha_3/\alpha_2)$ when $\xi=1$, and $K=\max(1,(\epsilon\alpha_3/\alpha_2)^{1/\xi},[2\xi/(\alpha_2\epsilon)]^{1/(1-\xi)})$ when $\xi\in (0,1)$. The main reason for introducing $K$ here is to make sure that $\epsilon_0\leq \alpha_2/\alpha_3$. We have the following result, whose proof is presented in Appendix \ref{pf:co_diminishing_step_size}.

\begin{corollary}\label{co:diminishing_step_size}
	Suppose $\epsilon_k$ is of the form given above, then the following bounds hold.
	\begin{enumerate}
		\item When $\xi=1$ and $\epsilon< 1/\alpha_2$, we have for all $k\geq 0$:
		\begin{align*}
			\mathbb{E}[\|x_k-x^*\|_c^2]\leq \alpha_1\|x_0-x^*\|_c^2\left(\frac{K}{k+K}\right)^{\alpha_2\epsilon}+\frac{4\epsilon^2\alpha_4}{1-\alpha_2\epsilon}\frac{A+2B\|x^*\|_c^2}{(k+K)^{\alpha_2\epsilon}}.
		\end{align*}
		\item When $\xi=1$ and $\epsilon= 1/\alpha_2$, we have for all $k\geq 0$:
		\begin{align*}
			\mathbb{E}[\|x_k-x^*\|_c^2]\leq\alpha_1\|x_0-x^*\|_c^2\frac{K}{k+K}+\frac{4\alpha_4}{\alpha_2^2}\frac{(A+2B\|x^*\|_c^2)\log(k+K)}{k+K}.
		\end{align*}
		\item When $\xi=1$ and $\epsilon> 1/\alpha_2$, we have for all $k\geq 0$:
		\begin{align*}
			\mathbb{E}[\|x_k-x^*\|_c^2]\leq\alpha_1\|x_0-x^*\|_c^2\left(\frac{K}{k+K}\right)^{\alpha_2\epsilon}+\frac{4e\epsilon^2\alpha_4}{\alpha_2\epsilon-1}\frac{A+2B\|x^*\|_c^2}{k+K}.
		\end{align*}
		\item When $\xi\in (0,1)$ and $\epsilon>0$, we have for all $k\geq 0$:
		\begin{align*}
			\mathbb{E}[\|x_k-x^*\|_c^2]\leq\alpha_1\|x_0-x^*\|_c^2\exp\left\{-\frac{\alpha_2\epsilon}{1-\xi}\left[(k+K)^{1-\xi}-\!K^{1-\xi}\right]\right\}+\frac{2\epsilon\alpha_4}{\alpha_2}\frac{A+2B\|x^*\|_c^2}{(k+K)^\xi}.
		\end{align*}
	\end{enumerate}
\end{corollary}

According to Corollary \ref{co:diminishing_step_size}, when the stepsizes are chosen as  $\epsilon_k=\epsilon/(k+K)$, the constant $\epsilon$ is important in determining the convergence rate, and the best convergence rate of $O(1/k)$ is attained when $\epsilon> 1/\alpha_2$. This is because the constant $\alpha_2$ (see Eq. (\ref{eq:finite_bound})) represents the real contraction effect of the algorithm. When $\alpha_2$ is small, we choose large $\epsilon$ to compensate for the slow contraction.

When $\xi\in (0,1)$, the convergence rate is roughly $O(1/k^\xi)$, which is sub-optimal but more robust, since the rate does not depend on the choice of $\epsilon$. This suggests the following rule of thumb in tuning the stepsizes. If we know the contraction factor $\gamma$, we know $\alpha_2$ given in Proposition \ref{prop:recursion} (since we pick $g(x)$ and $\mu$ ). Thus by choosing $\epsilon_k=\epsilon/(k+K)$ with $\epsilon>1/\alpha_2$, we achieve the optimal convergence rate. When our estimate of $\gamma$ is poor, to avoid being in case $1$ of Corollary \ref{co:diminishing_step_size}, it is better to use $\epsilon_k=\epsilon/(k+K)^\xi$ as the stepsize, thereby trading-off between convergence rate and robustness.

\textbf{Connection to SGD:} Although Theorem \ref{thm:sa-finite-time-bound} is derived for SA algorithms involving a contraction operator, they also recover finite-sample bounds for SGD with a smooth and strongly convex objective. To see this, let $F(x)$ be a differentiable objective function which is smooth and strongly convex with parameters $C$ and $c$. Define the operator $\mathcal{H}$ by $\mathcal{H}(x)=-\eta \nabla F(x)+x$, where $\eta>0$. Then Algorithm (\ref{sa:algorithm}) becomes $x_{k+1}=x_k+\epsilon_k (-\eta \nabla F(x_k)+w_k)$,
which is the SGD algorithm for minimizing $F(x)$ \cite{lan2020first,nemirovski2009robust}. Further, it is known that $\mathcal{H}$ is a Lipschitz operator w.r.t. the Euclidean norm $\|\cdot\|_2$, with Lipschitz constant $L_{SGD}=\max(|1-\eta c|,|1-\eta C|)$ \cite{ryu2016primer}. Therefore, when $\eta \in (0,2/C)$, we have $L_{SGD}<1$, and hence the operator $\mathcal{H}$ is a contraction with respect to $\|\cdot\|_2$. 

\textbf{Logarithmic Dependence on Dimension:} Switching focus, we next show in the following Corollary that with suitable choices of $g(x)$ (i.e. $\|\cdot\|_s$) and $\mu$, our approach naturally results in only logarithmic dependence on the dimension $d$ in the case where both $\|\cdot\|_c$ and $\|\cdot\|_e$ are the $\ell_\infty$-norm. The case of $\ell_\infty$-norm contraction is of special interest due to its applications in RL. The proof of the following corollary is presented in Appendix \ref{ap:tune_constants}.
\begin{corollary}\label{co:log-dependence}
	Consider the case where $\|\cdot\|_c=\|\cdot\|_e=\|\cdot\|_\infty$. Let $g(x)=\frac{1}{2}\|x\|_p^2$ with $p=4\log(d)$ and let $\mu=(1/2+1/(2\gamma))^2-1$. Then we have $\alpha_1\leq \frac{3}{2}$, $\alpha_2\geq \frac{1}{2}(1-\gamma)$, $\alpha_3\leq  \frac{32e(B+2)\log(d)}{1-\gamma}$, and $\alpha_4\leq \frac{16e\log(d)}{1-\gamma}$. 
\end{corollary}

\textbf{Order-Wise Tightness:} In general, we cannot hope to improve the convergence rate beyond $O(1/k)$ or the dimension dependence beyond $\log(d)$. To see this, consider the trivial case where $\mathcal{H}(x)$ is identically zero, and $\{w_k\}$ is an i.i.d. sequence of standard normal random vectors. Algorithm (\ref{sa:algorithm}) becomes $x_{k+1}=x_k+\epsilon_k(-x_k+w_k)$, which can be viewed as an SA algorithm for solving the trivial equation $x=0$, or an SGD algorithm for minimizing a quadratic objective $F(x)=\frac{1}{2}\|x\|_2^2$. When $\epsilon_k=\frac{1}{k+1}$, the iterates $x_k$ are just the running averages of $\{w_k\}$, i.e., $x_k=\frac{1}{k}\sum_{i=0}^{k-1}w_i$ for all $k\geq 1$, which implies $x_k\sim \frac{1}{\sqrt{k}}\mathcal{N}(0,I_d)$. It follows that $\mathbb{E}[\|x_k\|_\infty^2]=O(\frac{\log (d)}{k})$ \cite{vershynin2018high}. Thus in this setting, our resulting finite-sample bounds under $\ell_\infty$-norm contraction are order-wise tight both in terms of the convergence rate and the dimensional dependence.

\subsection{Convergence Bounds for Non-expansive Operators}

So far we have been considering contractive (or pseudo-contractive) operators, an immediate question is that what happens when the operator $\mathcal{H}$ is not contractive but only \textit{non-expansive}. In this subsection, we will consider non-expansive operators with respect to the Euclidean norm $\|\cdot\|_2$.
\begin{definition}\label{as:average1}
	The operator $\mathcal{H}$ is called non-expansive with respect to the norm $\|\cdot\|_2$ if and only if $\|\mathcal{H}(x)-\mathcal{H}(y)\|_2\leq \|x-y\|_2$ for all $x,y\in\mathbb{R}^d$.
\end{definition}

It turns out that for general non-expansive operators: (a) fixed-points need not exist (e.g. $\mathcal{H}(x)=x+1$), (b) fixed-points need not be unique even if they exist e.g. ($\mathcal{H}(x)=x$), and (c) fixed-point iteration $x_{k+1}=\mathcal{H}(x_k)$ need not converge to a fixed-point of $\mathcal{H}$ even if the set of fixed-point of $\mathcal{H}$, denoted by $\mathcal{X}$, is non-empty ($\mathcal{H}$ being the
reflection operator through a plane). 

One way to guarantee convergence is to perform fixed-point iteration with respect to the \textit{averaged}  operator, which is equivalent to a small stepsize variant of the original fixed-point iteration. The $\epsilon$-averaged operator of $\mathcal{H}$ is defined by $\mathcal{H}_{\epsilon}:=(1-\epsilon)I_d+\epsilon\mathcal{H}$. It is clear that $\mathcal{H}_\epsilon$ and $\mathcal{H}$ have the same set of fixed-points. Moreover, it was shown in \cite{ryu2016primer} that the fixed-point iteration $x_{k+1}=\mathcal{H}_\epsilon(x_k)$ converges to some fixed-point $x^*\in\mathcal{X}$, and the convergence rate is $O(1/k)$ in the sense that $\min_{0\leq i\leq k}\|\mathcal{H}_\epsilon(x_k)-x_k\|_2^2=O(1/k)$. 

Now we generalize this convergence bound to the stochastic setting in the following theorem, whose proof is presented in Appendix \ref{ap:pf:average}. Denote $\text{dist}(x,\mathcal{X}):=\inf_{x^*\in\mathcal{X}}\|x-x^*\|_2$ as the distance from $x$ to the set of fixed-points $\mathcal{X}$.
\begin{theorem}\label{thm:average}
	Consider $\{x_k\}$ generated by Algorithm (\ref{sa:algorithm}). Suppose that:
	\begin{enumerate}
		\item The operator $\mathcal{H}$ is non-expansive with respect to $\|\cdot\|_2$, and the set $\mathcal{X}$ of the fixed-points of $\mathcal{H}$ is non-empty and bounded.
		\item The noise sequence $\{w_k\}$ satisfies Assumption \ref{as:noise} with $B=0$.
	\end{enumerate}
	Then when the stepsize sequence $\{\epsilon_k\}$ satisfies $\sum_{k=0}^{\infty}\epsilon_k=\infty$ and $\sum_{k=0}^{\infty}\epsilon_k^2<\infty$, we have $\lim_{k\rightarrow \infty}\text{dist}(x_k,\mathcal{X})=0$ almost surely. Moreover, the following finite-sample bounds hold for all $k\geq 0$:
	\begin{align*}
		\min_{0\leq i\leq k}\mathbb{E}[\|\mathcal{H}(x_i)-x_i\|_2^2]\leq\begin{dcases}
			\frac{D^2}{(k+1)(1-\epsilon)\epsilon}+\frac{A\epsilon}{1-\epsilon},&\text{ when }\epsilon_k\equiv \epsilon\in (0,1),\\
			&\\
		   \frac{D^2+A\epsilon^2(1+\log(k))}{2(1-\epsilon)\epsilon((k+1)^{1/2}-1)},&\text{when }\epsilon_k=\frac{\epsilon}{\sqrt{k+1}} \text{ and }\epsilon\in (0,1),\\
		   &\\
		   \frac{D^2+2A\epsilon^2}{(1-\epsilon)\epsilon}\frac{1}{\log(k+1)},&\text{ when }\epsilon_k=\frac{\epsilon}{k+1} \text{ and } \epsilon\in (0,1),
		\end{dcases}
	\end{align*}
	where $D:=\sup_{x^*\in\mathcal{X}}\|x_0-x^*\|_2$.
\end{theorem}


Let us briefly compare Theorem \ref{thm:average} to Corollaries \ref{co:constant_step_size} and \ref{co:diminishing_step_size}, which are for contractive operators. When using constant stepsize, the convergence is geometric for contractive operators while $O(1/k)$ for non-expansive operators. When using $O(1/k)$ stepsizes, the convergence rate is $O(1/k)$ for contractive operators while $O(1/\log(k))$ for non-expansive operators. It is not a surprise to see that the convergence rates are worse in both cases. Note that the results of Theorem \ref{thm:average} recover the convergence bounds of SGD with a smooth but not strongly convex objective \cite{lan2020first} (which can be modeled by a fixed-point iteration with a non-expansive operator), where the convergence rate is $\Tilde{O}(1/\sqrt{k})$ when using $O(1/\sqrt{k})$ stepsizes.

\textbf{In summary},
we have (a) stated and proved a finite-sample error bound for Algorithm (\ref{sa:algorithm}) in its general form (Theorem \ref{thm:sa-finite-time-bound}), (b) studied its behavior under different choices of stepsizes (Corollaries \ref{co:constant_step_size} and \ref{co:diminishing_step_size}), (c) elaborated how to choose the function $g(x)$ and the parameter $\mu$ used in the generalized Moreau envelope to optimize the constants in the bound (\ref{eq:finite_bound}) (Corollary \ref{co:log-dependence}), and (d) extended our results to the case of non-expansive operators. In the next section, we present how the convergence results in this section apply to the context of RL.

\section{Applications in Reinforcement Learning}\label{sec:applications}

\subsection{Overview and Notation}

We study the infinite-horizon discounted (with discount factor $\beta\in (0,1)$) Markov Decision Process (MDP)  $\mathcal{M}=\{\mathcal{S},\mathcal{A},\mathcal{P},\mathcal{R}\}.$ Here, $\mathcal{S}$ is the finite state-space, $\mathcal{A}$ is the finite action-space,
$\mathcal{P}=\{P_a\in\mathbb{R}^{|\mathcal{S}|\times |\mathcal{S}|}\mid a\in\mathcal{A}\}$ is the set of unknown action dependent transition probability matrices, and  $\mathcal{R}:\mathcal{S}\times\mathcal{A}\mapsto \mathbb{R}$ is the reward function. Since we work with finite state-action spaces, we can without loss of generality assume $\mathcal{R}(s,a)\in [0,1]$ for all $(s,a)$. See \cite{sutton2018reinforcement} for more details about MDP. 

The ultimate goal in RL is to find a policy  $\pi^*$ (aka the optimal policy) that maximizes the expected total reward.
Specifically, the value of a policy $\pi$ at state $s$ is defined by $V_\pi(s)=\mathbb{E}_\pi[\sum_{k=0}^{\infty}\beta^k\mathcal{R}(S_k,A_k) \mid  S_0=s]$, where $\mathbb{E}_\pi[\;\cdot\;]$ means that the actions are executed according to the policy $\pi$. With value functions defined above, we want to find $\pi^*$ so that $V_{\pi^*}(s)\geq V_\pi(s)$ for any other policy $\pi$ and $s$.

The problem of directly finding an optimal policy is called the\textit{ control} problem. A sub-problem in RL is to estimate the value function $V_\pi$ of a given policy $\pi$, which is called the \textit{prediction} problem. The convergence of many classical algorithms in RL such as the TD-learning algorithm for solving the prediction problem and the $Q$-learning for solving the control problem relies on the stochastic approximation under contraction assumption \cite{bertsekas1996neuro}. Therefore, our result is a broad tool to establish the finite-sample error bounds of various RL algorithms. We next present three cases studies. Specifically, in the prediction setting, we establish convergence bounds of multi-step TD-learning algorithm using both on-policy \cite{sutton2018reinforcement} and off-policy sampling \cite{espeholt2018impala}. In the control setting, we recover the state-of-the-art results for the $Q$-learning algorithm \cite{watkins1992q}.

\subsection{The Prediction Problem}

In this section, we will use our results in Section \ref{sec:sa} to study the convergence bounds of TD-learning algorithm for solving the prediction problem. Specifically, in subsection \ref{subsec:vtrace}, we establish the first-known convergence bounds of a multi-step TD-learning algorithm using off-policy sampling (called V-trace \cite{espeholt2018impala}). In section \ref{subsec:TDn}, we derive convergence bounds for the on-policy $n$-step TD-learning algorithm (called TD$(n)$), and provide numerical example to support our results.

\subsubsection{The V-trace Algorithm for Off-Policy TD-Learning}\label{subsec:vtrace}

In off-policy TD-learning algorithms \cite{sutton2018reinforcement}, one uses trajectories generated by a \textit{behavior} policy $\pi'\neq \pi$ to learn the value function of the \textit{target} policy $\pi$. Off-policy methods are used for three important reasons in the TD-setting: (a) It is typically necessary to have an exploration component in the behavior policy $\pi'$ which makes it different from the target policy $\pi$. (b) It is used in multi-agent training where various agents collect rewards using a behavior policy that is lagging with respect to the target policy in an actor-critic framework \cite{espeholt2018impala}. (c) One set of samples can be used more than once to evaluate different target policies, which can leverage acquired data in the past.

Off-policy TD-learning is implemented through importance sampling to obtain an unbiased estimate of $V_\pi$. However, the variance in the estimate can blow up since the importance sampling ratio can be very large \cite{glynn1989importance}. Thus, a well-known and fundamental difficulty in off-policy TD-learning with importance sampling is that of balancing the bias-variance trade-off.

Recently, \cite{espeholt2018impala} proposed an off-policy TD-learning algorithm called the V-trace, where they introduced two truncation levels in the importance sampling weights. Their construction (through two separate clippers) crucially allows the algorithm to control the bias in the limit (through one clipper), while the other clipper mainly controls the variance in the estimate. The V-trace algorithm has had a huge practical impact: it has been implemented in distributed RL architectures and platforms like IMPALA (a Tensorflow implementation) \cite{espeholt2018impala} and TorchBeast (a PyTorch implementation) \cite{kuttler2019torchbeast} for multi-agent training besides being used at scale in a recent Deepmind City Navigation Project “Street Learn” \cite{mirowski2018learning}. Given its impact, a theoretical understanding of the effects of the truncation levels on convergence rate is important for determining how to tune them to improve the performance of V-trace.

In this paper we consider a synchronous version of the V-trace algorithm. Let $\pi'$ be a behavior policy used to collect samples, and let $\pi$ be the target policy whose value function is to be estimated. We first initialize $V_0\in\mathbb{R}^{|\mathcal{S}|}$. Given a fixed horizon $n\geq 1$, in each iteration $k\geq 0$, for each state $s\in\mathcal{S}$, a trajectory $\{S_0^s,A_0^s,...,S_{n-1}^s,A_{n-1}^s,S_n^s\}$ with initial state $S_0^s=s$ is generated using the behavior policy $\pi'$. Then, the corresponding entry of the estimate $V_k$ is updated according to
\begin{align}\label{vtrace_algorithm}
	V_{k+1}(s)=V_k(s)+\epsilon_k\sum_{t=0}^{n-1}\beta^t\left(\prod_{j =0}^{t-1}c_j^s\right)\rho_t^s \left(\mathcal{R}(S_t^s,A_t^s)+\beta V_k(S_{t+1}^s)-V_k(S_t^s)\right),
\end{align}
where $c_t^s=\min\left(\bar{c},\frac{\pi(A_t^s|S_t^s)}{\pi'(A_t^s|S_t^s)}\right)$ and $\rho_t^s=\min\left(\bar{\rho},\frac{\pi(A_t^s|S_t^s)}{\pi'(A_t^s|S_t^s)}\right)$ are truncated importance sampling weights with truncation levels $\bar{\rho}\geq \bar{c}$. Here we use the convention that $c_t^s=\rho_t^s=1$, and $\mathcal{R}(S_t^s,A_t^s)=0$ whenever $t<0$. In the special case where the behavior policy $\pi'$ and the target policy $\pi$ coincide, and $\bar{c}\geq 1$, Algorithm (\ref{vtrace_algorithm}) boils down to the on-policy multi-step TD-learning update \cite{sutton2018reinforcement}, which will be studied in subsection \ref{subsec:TDn}. The following assumption is needed to study Algorithm (\ref{vtrace_algorithm}).
\begin{assumption}\label{as:coverage}
	It holds for any $s\in\mathcal{S}$ that $\{a\mid \pi(a|s)>0\}\subseteq \{a\mid \pi'(a|s)>0\}$.
\end{assumption}
\begin{remark}
	Assumption \ref{as:coverage} is usually referred to as the \textit{coverage} assumption, which ensures that the behavior policy $\pi'$ is at least as explorative as the target policy $\pi$. This assumption is typically necessary in off-policy setting.
\end{remark}

The asymptotic convergence of Algorithm (\ref{vtrace_algorithm}) when $n=\infty$ has been established in \cite{espeholt2018impala} using the convergence results of stochastic approximation under contraction assumptions \cite{bertsekas1996neuro,kushner2010stochastic}. The quality of the V-trace limit as a function of $\bar{\rho}$ and $\bar{c}$ has also been discussed in \cite{espeholt2018impala}. Specifically, $\bar{\rho}$ alone determines the limiting value function, and $\bar{c}$ mainly controls the variance in the estimates $\{V_k\}$. Our goal is to understand the convergence rate of Algorithm (\ref{vtrace_algorithm}) for any choices of $\bar{\rho}$ and $\bar{c}$, which will determine the bias-vs-convergence-rate trade-off. First of all, when $n<\infty$, the following proposition summarizes the properties of the V-trace algorithm, whose proof is presented in Appendix \ref{appendix-vtrace}.
\begin{proposition}\label{prop:vtrace}
	The V-trace algorithm admits the following properties.
	\begin{enumerate}
		\item Algorithm (\ref{vtrace_algorithm}) can be equivalently written as $V_{k+1}=V_k+\epsilon_k(\mathcal{T}_n(V_k)-V_k+w_k)$, where $\mathcal{T}_n$ is an operator and $\{w_k\}$ is the noise sequence.
		\item Let $\kappa_c=\min_{s\in\mathcal{S}}\mathbb{E}_{\pi'}[c_0^s\mid S_0^s=s]$ and $\kappa_\rho=\min_{s\in\mathcal{S}}\mathbb{E}_{\pi'}[\rho_0^s\mid S_0^s=s]$. Then we have $0<\kappa_c\leq \kappa_\rho\leq 1$, and the operator $\mathcal{T}_n$ is a $\gamma$-contraction w.r.t. $\|\cdot\|_\infty$, where $\gamma=1-\frac{(1-\beta)(1-(\beta \kappa_c)^n)\kappa_\rho}{(1-\beta\kappa_c)}$.
		\item The operator $\mathcal{T}_n$ has a unique fixed-point $V_{\pi_{\bar{\rho}}}$, where $\pi_{\bar{\rho}}(a|s)=\frac{\min(\bar{\rho}\pi'(a|s),\pi(a|s))}{\sum_{b\in\mathcal{A}}\min(\bar{\rho}\pi'(b|s),\pi(b|s))}$ for all $(s,a)$, and we have $\|V_{\pi_{\bar{\rho}}}-V_\pi\|_\infty\leq \frac{2}{(1-\beta)^2}\|\pi_{\bar{\rho}}-\pi\|_\infty$.
		\item The noise sequence $\{w_k\}$ satisfies Assumption \ref{as:noise} with $\|\cdot\|_e$ being $\|\cdot\|_\infty$ and
		\begin{align*}
			A=B=\begin{dcases}
				\frac{32\bar{\rho}^2}{(1-\beta\bar{c})^2},&\text{when }\beta\bar{c}<1,\\
				32\bar{\rho}^2n^2,&\text{when }\beta\bar{c}=1,\\
				\frac{32\bar{\rho}^2(\beta\bar{c})^{2n}}{(\beta\bar{c}-1)^2},&\text{when }\beta\bar{c}>1.
			\end{dcases}
		\end{align*}
	\end{enumerate}
\end{proposition}
\begin{remark}
	Due to the truncation level $\bar{\rho}$ introduced in Algorithm (\ref{vtrace_algorithm}), the fixed-point $V_{\pi_{\bar{\rho}}}$ of $\mathcal{T}_n$ is not necessarily the target value function $V^\pi$. When $\bar{\rho}\geq\rho_{\max}:= \max_{(s,a)}\frac{\pi(a|s)}{\pi'(a|s)}$, we have $\pi_{\bar{\rho}}=\pi$. Otherwise the policy $\pi_{\bar{\rho}}$ is in some sense between the behavior policy $\pi'$ and the target policy $\pi$. 
\end{remark}

\paragraph{Finite-Sample Analysis of the V-trace Algorithm}
From Proposition \ref{prop:vtrace}, we see that our results in Section \ref{sec:sa} are applicable. Moreover, since $\|\cdot\|_c=\|\cdot\|_e=\|\cdot\|_\infty$ in this problem, we can use Corollary \ref{co:log-dependence} for the choices of the smoothing norm $\|\cdot\|_s$ and the parameter $\mu$. For ease of exposition, here we only consider the $O(1/k)$ stepsizes, and pick the parameters to ensure that we fall in case 3 of Corollary \ref{co:diminishing_step_size}, which has the best convergence rate. The finite-sample error bound for other cases can be derived similarly. The proof of the followng theorem is presented in Appendix \ref{pf:thm:vtrace}.
\begin{theorem}\label{thm:vtrace}
	Consider $\{V_k\}$ of Algorithm (\ref{vtrace_algorithm}). Suppose that $\epsilon_k=\frac{\epsilon}{k+K}$ with $	\epsilon=\frac{4}{1-\gamma}$ and $K=\frac{64(A+2)\log |\mathcal{S}|}{(1-\gamma)^3}$. Then we have for all $k\geq 0$:
	\vspace{-0.4em}
	\begin{align}\label{bound_vtrace}
		\mathbb{E}\left[\|V_k-V_{\pi_{\bar{\rho}}}\|_\infty^2\right]
		\leq 1024e^2(\|V_0-V_{\pi_{\bar{\rho}}}\|_\infty^2+2\|V_{\pi_{\bar{\rho}}}\|_\infty^2+1)\frac{ (A+2)\log |\mathcal{S}|}{(1-\gamma)^3}\frac{1}{k+K}.
	\end{align}
\end{theorem}

To better understand the how the parameters $\bar{c}$, $\bar{\rho}$, and $n$ impact the convergence rate. Suppose we want to find the number of iterations so that in expectation the distance between $x_k$ and $x^*$ is less than $\delta$, i.e.,  $k_{\delta}=\min\left\{k\geq 0:\mathbb{E}[\|x_k-x^*\|_\infty^2]\leq \delta\right\}$. Using Eq. (\ref{bound_vtrace}) and we have $k_\delta\geq 1024 e^2  (\|V_0-V_{\pi_{\bar{\rho}}}\|_\infty^2+2\|V_{\pi_{\bar{\rho}}}\|_\infty^2+1)\frac{(A+2) \log|\mathcal{S}|}{\delta(1-\gamma)^3}$. We first note that the dimension dependence of $k_\delta$ is only $\log (|\mathcal{S}|)$. The parameters $\bar{c}$, $\bar{\rho}$, and $n$ impact the convergence rate through the constant $A$. 

From property $(d)$ of Propositon \ref{prop:vtrace}, we see that $A$ is a piecewise function of $\bar{c}$, $\bar{\rho}$, and $n$. In all its cases, $\bar{\rho}$ appears quadratically in $A(\bar{c},\bar{\rho},n)$. The impact of $\bar{c}$ and $n$ is more subtle. When $\beta\bar{c}<1$, $A(\bar{c},\bar{\rho},n)$ is independent of the horizon $n$. However, when $\beta\bar{c}=1$ or $\beta\bar{c}>1$, $A(\bar{c},\bar{\rho},n)$ increases either linearly or exponentially in terms of $n$, which suggests that $\bar{c}<1/\beta$ is a better choice. Such a small  $\bar{c}$ can lead to the contraction factor $\gamma$ being close to unity (Proposition \ref{prop:vtrace} $(b)$), which increases the error in Eq. (\ref{bound_vtrace}). However, since $A$ does not depend on $n$ when $\bar{c}<1/\beta$, this drawback can be avoided by increasing the horizon $n$, which decreases the contraction parameter $\gamma$, albeit at the cost of more samples in each iterate. For a given application, based on the above idea, we can numerically optimize the parameters ($\bar{c}$, $\bar{\rho}$, and $n$) to trade-off between contraction factor and variance.

Though we have analyzed the convergence rate of $V_k$, the limiting value function $V_{\pi_{\bar{\rho}}}$ is not the value function of the target policy $\pi$. Note that this bias can be eliminated by choosing $\bar{\rho}\geq \rho_{\max}$, provided that $\rho_{\max}$ is finite. However, when the number of state-action pairs is infinite, and when we use V-trace algorithm along with function approximation, $\rho_{\max}$ can be infinity. Studying such a scenario is one of our future directions.

\subsubsection{The On-Policy TD($n$) Algorithm}\label{subsec:TDn}
Recall from the previous section that when the behavior policy and the target policy are the same and $\bar{c}=\bar{\rho}\geq 1$, the V-trace algorithm (\ref{vtrace_algorithm}) reduces to the on-policy $n$-step TD-learning algorithm (with synchronous update):
\begin{align}\label{algo:TDn}
	V_{k+1}(s)=V_k(s)+\epsilon_k\left(\sum_{t=0}^{n-1}\beta^{t}\mathcal{R}(S_t^s,A_t^s)+\beta^n V_k(S_{n}^s)-V_k(s)\right).
\end{align}

Note that since the TD$(n)$ algorithm is a special case of the V-trace algorithm, it possesses properties given in Proposition \ref{prop:vtrace} (e.g., $\|\cdot\|_\infty$-norm contraction). Here to derive convergence bounds, we will use a slightly different approach. Specifically, we will show that instead of only being a contraction with respect to the $\ell_\infty$-norm, the corresponding operator of Algorithm (\ref{algo:TDn}) is also a contraction with respect to a weighted $\ell_2$-norm. To establish such property, the following assumption is needed.

\begin{assumption}\label{as:TDn}
	The Markov chain $\{S_k\}$ under policy $\pi$ has a unique stationary distribution $\lambda\in\Delta^{|\mathcal{S}|}$, and $\lambda(s)>0$ for all $s\in\mathcal{S}$.
\end{assumption}
\begin{remark}
	In order to use on-policy sampling, the target policy itself must be explorative, which is a consequence of Assumption \ref{as:TDn}. One condition that ensures Assumption \ref{as:TDn} is that the finite-state Markov chain $\{S_k\}$ is irreducible and aperiodic \cite{levin2017markov}, and is used in many related works \cite{tsitsiklis1997analysis,tsitsiklis1999average}.
\end{remark}

Let $\Lambda\in\mathbb{R}^{|\mathcal{S}|\times|\mathcal{S}|}$ be a diagonal matrix with diagonal entries being $\{\lambda(s)\}_{s\in\mathcal{S}}$. Let $\langle \cdot,\cdot\rangle_{\Lambda}$ be an inner product in $\mathbb{R}^{|\mathcal{S}|}$ defined by $\langle V_1,V_2\rangle_{\Lambda}=V_1^\top \Lambda V_2$. Let $\|\cdot\|_{\Lambda}$ be the norm induced by $\langle \cdot,\cdot\rangle_{\Lambda}$, i.e., $\|V\|_\Lambda=(V^\top \Lambda V)^{1/2}$. We next present the $\|\cdot\|_\Lambda$-norm contraction property of Algorithm (\ref{algo:TDn}) in the following proposition, whose proof is provided in Appendix \ref{ap:pf:proposition_TDn}

\begin{proposition}\label{prop:TDn}
	The following properties hold for the TD$(n)$ algorithm (\ref{algo:TDn}).
	\begin{enumerate}
		\item Algorithm (\ref{algo:TDn}) can be equivalently written  as $V_{k+1}=V_k+\epsilon_k(\mathcal{T}_n(V_k)-V_k+w_k)$, where $\mathcal{T}_n$ is an operator and $w_k$ is the noise.
		\item The operator $\mathcal{T}_n$ is a $\beta^n$-contraction with respect to $\|\cdot\|_\Lambda$.
		\item The random process $\{w_k\}$ satisfies Assumption \ref{as:noise} with $\|\cdot\|_e$ being $\|\cdot\|_\Lambda$, $A=\frac{2(1-\beta^n)^2}{(1-\beta)^2}$, and $B=2\beta^{2n}$.
	\end{enumerate}
\end{proposition}

The advantage of $\|\cdot\|_\Lambda$-norm contraction is that the norm squared function $f(x)=\frac{1}{2}\|x\|_\Lambda^2$ itself is a $1$-smooth function with respect to $\|\cdot\|_\Lambda$. Therefore, there is no need to modify $f(x)$, or in other words, when constructing the Lyapunov function $M(x)$, we can simply choose $\mu=1$ and the smoothing norm $\|\cdot\|_s$ to be $\|\cdot\|_\Lambda$. Next, using Theorem \ref{thm:sa-finite-time-bound} on the TD$(n)$ algorithm and we have the following finite-sample bounds. For ease of exposition, we only state the result for the case of constant stepsizes. See Appendix \ref{ap:pf:TDn} for the proof.

\begin{theorem}\label{thm:TDn}
	Consider $\{V_k\}$ generated by the on-policy TD$(n)$ algorithm (\ref{algo:TDn}). Suppose that $\epsilon_k\equiv \epsilon\leq\frac{1-\beta^n}{16(1+\beta^{2n})}$. Then we have for all $k\geq 0$:
	\begin{align}\label{bound_TDn}
		\mathbb{E}[\|V_k-V_\pi\|_\Lambda^2]=\|V_0-V_\pi\|_\Lambda^2(1-(1-\beta^n)\epsilon)^k+\frac{8}{1-\beta^n}\left(\frac{(1-\beta^n)^2}{(1-\beta)^2}+2\beta^{2n}\|V_\pi\|_\Lambda^2\right)\epsilon.
	\end{align}
\end{theorem}
\begin{remark}
	It is well-known that TD$(n)$ algorithm can be viewed as a linear stochastic approximation algorithm involving a Hurwitz matrix \cite{bertsekas1996neuro,tsitsiklis1997analysis}. This suggests a different way of deriving convergence bounds, where Lyapunov functions are constructed using Lyapunov equations. However, one can show that the two approaches are similar in a way, and they indeed give same finite-sample convergence bounds (\ref{bound_TDn}).
\end{remark}

Note that according to Theorem \ref{thm:TDn}, it is clear that the convergence rate of the first term on the RHS of Eq. (\ref{bound_TDn}) is in favor of large $n$, while under some mild conditions the second term is in favor of small $n$. This observation indicates that, for large $n$, the bias in the estimate $V_k$ will fastly be eliminated, but the variance is relatively large. However, since the second term on the RHS of Eq. (\ref{bound_TDn}) can be further bounded above by $\frac{32}{(1-\beta)^2}+\frac{16\beta^2}{1-\beta}\|V_\pi\|_{\Lambda}^2$, the impact of the variance is somewhat limited. The following numerical example is used to support our analysis above.

\begin{figure}
\centering
	\includegraphics[width=0.55\linewidth]{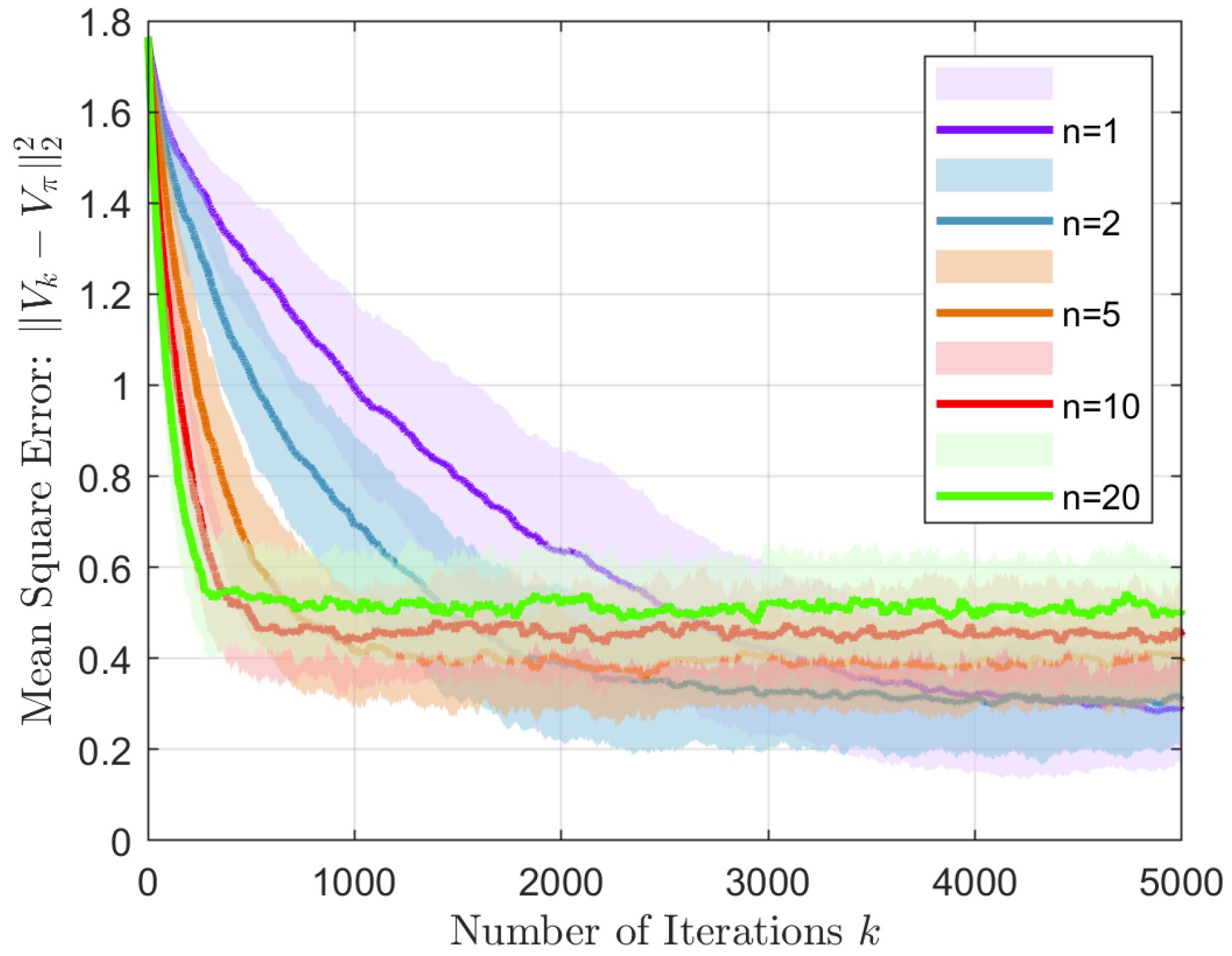}
    \caption{The convergence rates of TD($n$) algorithms\label{fig1}}
\end{figure}

\textbf{Example:} We consider an MDP with $10$ states and $3$ actions. The transition probabilities and the reward function are both randomly generated. The target policy is to take each action with equal probability regardless of the current state. The value function $V_\pi$ is calculated by solving the Bellman's equation. In Figure \ref{fig1}, we plot the mean square distance (with error bar) between $V_k$ and $V_\pi$ as a function of the number of iterations for different values of $n$, where each line is obtained by simulating $100$ sample paths. We see that, in general, when $n$ increases, the initial convergence rate is better, but asymptotically the mean square error stays at a higher value. However, when $n$ is already large, the asymptotic error does not increase much when increasing $n$ (see the plots for $n=10$ and $n=20$).

\subsection{The Control Problem}\label{subsec;Q-learning}

Recall that in the control problem we aim at finding an optimal policy $\pi^*$, and a popular algorithm to achieve that is the $Q$-learning algorithm \cite{watkins1992q}. We begin by defining the $Q$-function associated with an optimal policy $\pi^*$ by: $Q^*(s,a)=\mathbb{E}_{\pi^*}\left[\sum_{k=0}^{\infty}\beta^k\mathcal{R}(S_k,A_k)\;\middle|\;S_0=s,A_0=a\right]$ for all $(s,a)$. It turns out that the following relation holds between the optimal policy $\pi^*$ and the optimal $Q$-function $Q^*$: $\pi^*(s)\in \arg\max_{a\in\mathcal{A}}Q^*(s,a)$ for all state $s$. Moreover, $Q^*$ unique solves the following Bellman's optimality equation: $Q=\mathcal{T}(Q)$, where $\mathcal{T}$ is the Bellman's optimality operator defined by
\begin{align*}
	[\mathcal{T}(Q)](s,a)=\mathcal{R}(s,a)+\beta\mathbb{E}\left[\max_{a'\in\mathcal{A}}Q(S_{k+1},a')\;\middle|\;S_k=s,A_k=a\right],\quad \forall\;(s,a)\in\mathcal{S}\times\mathcal{A}.
\end{align*}
Therefore, solving the control problem reduces to solving the equation $Q=\mathcal{T}(Q)$, which leads to the $Q$-learning algorithm. 

Consider the following $Q$-learning algorithm (in the synchronous setting) of \cite{watkins1992q}: first initialize $Q_0\in\mathbb{R}^{|\mathcal{S}||\mathcal{A}|}$, then at each time step, sample from each state-action pair $(s,a)$ its successor state $s'_{s,a}$, and update the corresponding component of the iterate $Q_k$ according to
\begin{align}\label{algo:Q-learning}
	Q_{k+1}(s,a)=Q_k(s,a)+\epsilon_k\left(\mathcal{R}(s,a)+\beta\max_{a'\in\mathcal{A}}Q_k(s'_{s,a},a')-Q_k(s,a)\right).
\end{align}

To apply our results in Section \ref{sec:sa}, we first summarize the properties of the $Q$-learning algorithm in the following proposition.
\begin{proposition}\label{prop:synchronous_Q}
	The $Q$-learning algorithm (\ref{algo:Q-learning}) admits the following properties.
	\begin{enumerate}
		\item Algorithm (\ref{algo:Q-learning}) can be equivalently written as $Q_{k+1}=Q_k+\epsilon_k(\mathcal{T}(Q_k)-Q_k+w_k)$.
		\item The Bellman's optimality operator $\mathcal{T}$ is a $\beta$ -- contraction mapping w.r.t. $\|\cdot\|_\infty$.
		\item The noise $\{w_k\}$ satisfies Assumption \ref{as:noise} with $\|\cdot\|_e$ being $\|\cdot\|_\infty$ and $A=B=8$.
	\end{enumerate}
\end{proposition}

Proposition \ref{prop:synchronous_Q} indicates that Theorem \ref{thm:sa-finite-time-bound} is applicable for $Q$-learning. To compare our result with existing literature \cite{beck2012error,wainwright2019stochastic}, we will apply Corollary \ref{co:constant_step_size} and Corollary \ref{co:diminishing_step_size} case (c) (which gives the optimal asymptotic rate) to obtain the finite-sample error bounds for $Q$-learning.
\begin{theorem}\label{thm:Q-learning}
	Consider the iterates $\{Q_k\}$ generated by the $Q$-learning Algorithm (\ref{algo:Q-learning}). 
	\begin{enumerate}
		\item Suppose that $\epsilon_k=\epsilon\leq \frac{(1-\beta)^2}{640e\log (|\mathcal{S}||\mathcal{A}|)}$ for all $k\geq 0$. Then we have for all $k\geq 0$:
		\begin{align*}
			\mathbb{E}\left[\|Q_k-Q^*\|_\infty^2\right]\leq \frac{3}{2}\|Q_0-Q^*\|_\infty^2\left[1-\frac{(1-\beta)\epsilon}{2}\right]^k+(1+2\|Q^*\|_\infty^2)\frac{256e  \log (|\mathcal{S}||\mathcal{A}|)\epsilon}{(1-\beta)^2}.
		\end{align*}
		\item Suppose that $\epsilon_k=\epsilon/(k+K)$ with $\epsilon=\frac{4}{1-\beta}$ and $K=\frac{640e \log(|\mathcal{S}||\mathcal{A}|)}{(1-\beta)^3}$. Then we have for all $k\geq 0$:
		\begin{align*}
			\mathbb{E}\left[\|Q_k-Q^*\|_\infty^2\right]\leq 8192e^2(1+2\|Q^*\|_\infty^2+\|Q_0-Q^*\|_\infty^2)\frac{\log (|\mathcal{S}||\mathcal{A}|)}{(1-\beta)^3}\frac{1}{k+K}.
		\end{align*}
	\end{enumerate}
\end{theorem}

Theorem \ref{thm:Q-learning} (a) agrees with \cite{beck2012error} (Theorem 2.1) in that the iterates $\{Q_k\}$ converge exponentially fast to a ball centered at $Q^*$ with radius proportional to the stepsize $\epsilon$. However, with our approach, we get the dimensional dependence that scales as $\log (|\mathcal{S}||\mathcal{A}|)$ as compared to $(|\mathcal{S}||\mathcal{A}|)^2$ in \cite{beck2012error}.

Theorem \ref{thm:Q-learning} (b) is similar to Corollary 3 of \cite{wainwright2019stochastic} where the  dimension dependence appears as $\log (|\mathcal{S}||\mathcal{A}|)$ in the bound and the rate of convergence is $O(1/k)$. However, to derive such result, besides a similar contraction property of $\mathcal{T}$, \cite{wainwright2019stochastic} also used the monotonicity property of $\mathcal{T}$, and the fact that the iterates of $Q$-learning are uniformly bounded. Therefore, our approach is more general in that we need only the contraction property, and weaker noise assumptions.

\textbf{In summary}, we have (a) established the first-known finite-sample error bound of the V-trace algorithm using our general results on SA in Section \ref{sec:sa}, (b) analyzed how the parameters of the problem (i.e., the two truncation levels $\bar{c},\bar{\rho}$, and the horizon $T$) impact the convergence rate, and provided a rule of thumb in tuning them, (c) studied the convergence bounds of the TD$(n)$ algorithm and verified the result in a numerical example, and (d) studied the convergence bounds of the $Q$-learning algorithm, where we recovered the results in  \citep{wainwright2019stochastic} in a diminishing stepsize regime, and improve over \citep{beck2012error,beck2013improved} in a constant stepsize regime.

\section{Conclusion}
In this paper, we establish finite-sample bounds for stochastic approximation algorithms involving a contraction operator with respect a general norm. We prove this result using a novel Lyapunov function. Such a a smooth Lyapunov function is constructed using the generalized Moreau envelope, which involves the infimal convolution with respect to the square of some other suitable norm. By carefully choosing this norm based on the application, we are able to show that our approach provides bounds that only scale logarithmically in the dimension. Furthermore, we provide the first-known finite-sample analysis of the popular off-policy Reinforcement Learning V-trace algorithm, study the performance of the on-policy TD$(n)$ algorithm for different choices of $n$, and recover the existing state-of-the-art results for $Q$-learning.

\section*{Acknowledgement}
This work was partially supported by ONR Grant N00014-19-1-2566, NSF Grant CNS-1910112, ARO grant W911NF-17-1-0359, and Raytheon Technologies. Maguluri also acknowledges seed funding from Georgia Institute of Technology.

\bibliographystyle{apalike}
\bibliography{references}

\newpage

\appendix
\appendixpage

\section{Proofs of All Technical Results in Section \ref{sec:sa}}
	\subsection{Proof of Lemma \ref{le:properties_moreau}}\label{pf:le_properties_moreau}
	\begin{enumerate}
		\item The convexity of $M_f^{\mu,g}(x)$ follows from Theorem 2.19 of \cite{beck2017first}.
		Since $f(x)$ is proper, closed, and convex, $g(x)$ is $L$ -- smooth with respect to $\|\cdot\|_s$, we have by \cite{beck2017first} Theorem 5.30 (a) that $M_f^{\mu,g}(x)=(f\square \frac{g}{\mu})(x)$ is $\frac{L}{\mu}$ -- smooth with respect to $\|\cdot\|_s$.
		\item We first derive the upper bound of $f(x)$. By definition of $	M_f^{\mu,g}(x)$, we have
		\begin{align*}
			M_f^{\mu,g}(x)&=\min_{u\in\mathbb{R}^d}\left\{\frac{1}{2}\|u\|_c^2+\frac{1}{2\mu}\|x-u\|_s^2\right\}\\
			&\geq \min_{u\in\mathbb{R}^d}\left\{\frac{1}{2}\|u\|_c^2+\frac{\ell_{cs}^2}{2\mu}\|x-u\|_c^2\right\}\tag{$\ell_{cs}\|\cdot\|_c\leq \|\cdot\|_s$}\\
			&\geq \min_{u\in\mathbb{R}^d}\left\{\frac{1}{2}\|u\|_c^2+\frac{\ell_{cs}^2}{2\mu}(\|x\|_c-\|u\|_c)^2\right\}\tag{Triangle inequality}\\
			&=\min_{y\in\mathbb{R}}\left\{\frac{1}{2}y^2+\frac{\ell_{cs}^2}{2\mu}(\|x\|_c-y)^2\right\}\tag{change of variable: $y=\|u\|_c^2$}\\
			&=\min_{y\in\mathbb{R}}\left\{\left(\frac{1}{2}+\frac{\ell_{cs}^2}{2\mu}\right)y^2-\frac{\ell_{cs}^2}{\mu}\|x\|_cy+\frac{\ell_{cs}^2}{2\mu}\|x\|_c^2\right\}\\
			&=\frac{1}{2}\|x\|_c^2\frac{\ell_{cs}^2}{\mu+\ell_{cs}^2}\tag{minimum of a quadratic function}\\
			&=\frac{\ell_{cs}^2}{\mu+\ell_{cs}^2}f(x).
		\end{align*}
		It follows that $f(x)\leq \left(1+\mu/\ell_{cs}^2\right) M_f^{\mu,g}(x)$ for all $x$. Next we show the lower bound. Similarly, by definition we have
		\begin{align*}
			M_f^{\mu,g}(x)&=\min_{u\in\mathbb{R}^d}\left\{\frac{1}{2}\|u\|_c^2+\frac{1}{2\mu}\|x-u\|_s^2\right\}\\
			&\leq \min_{\alpha\in(0,1)}\left\{\frac{1}{2}\|\alpha x\|_c^2+\frac{1}{2\mu}\|x-\alpha x\|_s^2\right\}\tag{restrict $u=\alpha x$ for $\alpha\in (0,1)$}\\
			&\leq \frac{1}{2}\|x\|_c^2\min_{\alpha\in(0,1)}\left\{\alpha^2+\frac{(1-\alpha)^2u_{cs}^2}{\mu}\right\}\tag{$\|\cdot\|_s\leq u_{cs}\|\cdot\|_c$}\\
			&=\frac{u_{cs}^2}{u_{cs}^2+\mu}\frac{1}{2}\|x\|_c^2\tag{minimum of the quadratic function}\\
			&= \frac{u_{cs}^2}{u_{cs}^2+\mu}f(x).
		\end{align*}	
		It follows that $f(x)\geq \left(1+\mu/u_{cs}^2\right)M_f^{\mu,g}(x)$ for all $x$.
		\item It is clear from the definition of $M_f^{\mu,g}(x)$ that it is non-negative and is equal to zero if and only if $x=0$. Now for any $\alpha\in\mathbb{R}$, we have
		\begin{align*}
			M_f^{\mu,g}(\alpha x)&=\min_{u}\left\{\frac{1}{2}\|u\|_c^2+\frac{1}{2\mu}\|\alpha x-u\|_s^2\right\}\\
			&= \min_{v}\left\{\frac{1}{2}\|\alpha v\|_c^2+\frac{1}{2\mu}\|\alpha x-\alpha v\|_s^2\right\}\tag{change of variable $u=\alpha v$}\\
			&=|\alpha|^2M_f^{\mu,g}(x).
		\end{align*}
		Thus, $\sqrt{M_f^{\mu,g}(\alpha x)}=|\alpha|\sqrt{M_f^{\mu,g}(x)}$. It remains to show the Triangle inequality. For any $x_1,x_1\in\mathbb{R}^d$, let $u_1\in\arg\min_{u\in\mathbb{R}^d}\{\frac{1}{2}\|u\|_c^2+\frac{1}{2\mu}\|x_1-u\|_s^2\}$ and $u_2\in\arg\min_{u\in\mathbb{R}^d}\{\frac{1}{2}\|u\|_c^2+\frac{1}{2\mu}\|x_2-u\|_s^2\}$. Then we have
		\begin{align*}
			&M_f^{\mu,g}(x_1+x_2)\\
			=\;&\min_{u}\left\{\frac{1}{2}\|u\|_c^2+\frac{1}{2\mu}\|x_1+x_2-u\|_s^2\right\}\\
			\leq \;& \frac{1}{2}\|u_1+u_2\|_c^2+\frac{1}{2\mu}\|x_1+x_2-u_1-u_2\|_s^2\tag{choose $u=u_1+u_2$}\\
			\leq \;&\frac{1}{2}(\|u_1\|_c+\|u_2\|_c)^2+\frac{1}{2\mu}(\|x_1-u_1\|_s+\|x_2-u_2\|_s)^2\\
			=\;&M_f^{\mu,g}(x_1)+M_f^{\mu,g}(x_2)+ \|u_1\|_c\|u_2\|_c+\frac{1}{\mu}\|x_1-u_1\|_s\|x_2-u_2\|_s\\
			\leq\;& M_f^{\mu,g}(x_1)+M_f^{\mu,g}(x_2)+2\sqrt{\frac{1}{2}\|u_1\|_c^2+\frac{1}{2\mu}\|x_1-u_1\|_s^2}\sqrt{\frac{1}{2}\|u_2\|_c^2+\frac{1}{2\mu}\|x_2-u_2\|_s^2}\\
			=\;&M_f^{\mu,g}(x_1)+M_f^{\mu,g}(x_2)+2\sqrt{M_f^{\mu,g}(x_1)M_f^{\mu,g}(x_2)}.
		\end{align*}
		It follows that $\sqrt{M_f^{\mu,g}(x_1+x_2)}\leq \sqrt{M_f^{\mu,g}(x_1)}+\sqrt{M_f^{\mu,g}(x_2)}$ for any $x_1,x_2\in\mathbb{R}^d$.
		Therefore, we can write $M_f^{\mu,g}(x)$ as $\frac{1}{2}\|x\|_M^2$ for some norm $\|\cdot\|_M$.
	\end{enumerate}
	
	\subsection{Proof of Proposition \ref{prop:recursion}}\label{pf:le_recursion}
	Using Lemma \ref{le:properties_moreau} (a) and Algorithm (\ref{sa:algorithm}), we have
	\begin{align*}
		M_f^{\mu,g}(x_{k+1}-x^*)\leq \;&M_f^{\mu,g}(x_{k}-x^*)+\langle \nabla M_f^{\mu,g}(x_k-x^*),x_{k+1}-x_k\rangle+\frac{L}{2\mu}\|x_{k+1}-x_k\|_s^2\\
		= \;& M_f^{\mu,g}(x_{k}-x^*)+\epsilon_k\langle \nabla M_f^{\mu,g}(x_k-x^*),\mathcal{H}(x_k)-x_k\rangle\\
		&+\epsilon_k\langle \nabla M_f^{\mu,g}(x_k-x^*),w_k\rangle+\frac{L\epsilon_k^2}{2\mu}\|\mathcal{H}(x_k)-x_k+w_k\|_s^2.
	\end{align*}
	Taking expectation conditioned on $\mathcal{F}_k$ on both side of the previous inequality then using Assumption \ref{as:noise} (a), we have
	\begin{align}
		\mathbb{E}[M_f^{\mu,g}(x_{k+1}-x^*)\mid\mathcal{F}_k]
		\leq \;& M_f^{\mu,g}(x_{k}-x^*)+\epsilon_k\underbrace{\langle \nabla M_f^{\mu,g}(x_k-x^*),\mathcal{H}(x_k)-x_k\rangle}_{E_1}\nonumber\\
		&+\frac{L\epsilon_k^2}{2\mu}\underbrace{\mathbb{E}[\|\mathcal{H}(x_k)-x_k+w_k\|_s^2\mid\mathcal{F}_k]}_{E_2}\label{eq:expansion}.
	\end{align}
	We first control the term $E_1$ in the following. Using the fact that $\mathcal{H}(x^*)=x^*$, we have
	\begin{align}\label{eq:split-E-terms}
		E_1=\underbrace{\langle \nabla M_f^{\mu,g}(x_k-x^*),\mathcal{H}(x_k)-\mathcal{H}(x^*)\rangle}_{E_{1,1}}-\underbrace{\langle \nabla M_f^{\mu,g}(x_k-x^*),x_k-x^*\rangle}_{E_{1,2}}.
	\end{align}
	For the gradient of $M_f^{\mu,g}(x)$, since $M_f^{\mu,g}(x)=\frac{1}{2}\|x\|_M^2$, we have by the chain rule of subdifferential calculus (Theorem 3.47 of \cite{beck2017first}) that $\nabla M_f^{\mu,g}(x)=\|x\|_M v_{x}$, where $v_{x}\in \partial \|x\|_M$ is a subgradient of the function $\|x\|_M$ at $x$. In fact, from the equation $\nabla M_f^{\mu,g}(x)=\|x\|_M v_{x}$, we see that $v_x$ is unique (i.e., $v_x=\nabla \|x\|_M$) for all $x\neq 0$.
	
	Now consider the term $E_{1,1}$. Using H\"{o}lder's inequality, we have
	\begin{align}
		E_{1,1}&=\|x_k-x^*\|_M\langle v_{x_k-x^*},\mathcal{H}(x_k)-\mathcal{H}(x^*)\rangle\nonumber\\
		&\leq \|x_k-x^*\|_M\|v_{x_k-x^*}\|_M^*\|\mathcal{H}(x_k)-\mathcal{H}(x^*)\|_M,\label{eq:dual-norm}
	\end{align}
	where $\|\cdot\|_M^*$ is the dual norm of $\|\cdot\|_M$. To further control $E_{1,1}$, the following result is needed.
	
	\begin{lemma}[Lemma 2.6 of \cite{shalev2012online}]\label{le:Lipschitz}
		Let $h: \mathcal{D} \rightarrow \mathbb{R}$ be a convex function. Then $h$ is $L$ -- Lipschitz over $\mathcal{D}$ with respect to norm $\|\cdot\|$ if and only if for all $w\in \mathcal{D}$ and $z\in\partial h(w)$ we have that $\|z\|_{*}\leq L$, where $\|\cdot\|_{*}$ is the dual norm of $\|\cdot\|$.
	\end{lemma}
	
	Since $\|x\|_{M}$ as a function of $x$ is $1$ -- Lipschitz w.r.t. $\|\cdot\|_M$, we have by Lemma \ref{le:Lipschitz} that $\|v_{x_k-x^*}\|_M^*\leq 1$. For the term $\|\mathcal{H}(x_k)-\mathcal{H}(x^*)\|_M$ in Eq. (\ref{eq:dual-norm}), using Lemma \ref{le:properties_moreau} (b) and the contraction of $\mathcal{H}$ with respect to $\|\cdot\|_c$, we have
	\begin{align*}
		\frac{1}{2}\|\mathcal{H}(x_k)-\mathcal{H}(x^*)\|_M^2&=M_f^{\mu,g}(\mathcal{H}(x_k)-\mathcal{H}(x^*))\\
		&\leq \frac{1}{1+\mu/u_{cs}^2}f(\mathcal{H}(x_k)-\mathcal{H}(x^*))\tag{Lemma \ref{le:properties_moreau} (b)}\\
		&\leq \frac{\gamma^2}{1+\mu/u_{cs}^2}f(x_k-x^*)\tag{Assumption \ref{as:contraction}}\\
		&\leq \gamma^2\frac{1+\mu/\ell_{cs}^2}{1+\mu/u_{cs}^2}M_f^{\mu,g}(x_k-x^*)\tag{Lemma \ref{le:properties_moreau} (b)}\\
		&=\frac{\gamma^2}{2}\frac{1+\mu/\ell_{cs}^2}{1+\mu/u_{cs}^2}\|x_k-x^*\|_M^2,
	\end{align*}
	which implies
	\begin{align*}
		\|\mathcal{H}(x_k)-\mathcal{H}(x^*)\|_M\leq \gamma\left(\frac{1+\mu/\ell_{cs}^2}{1+\mu/u_{cs}^2}\right)^{1/2}\|x_k-x^*\|_M.
	\end{align*}
	Substituting the upper bounds we obtained for $\|v_{x_k-x^*}\|_M^*$ and $\|\mathcal{H}(x_k)-\mathcal{H}(x^*)\|_{M}$ into Eq. (\ref{eq:dual-norm}), we have
	\begin{align*}
		E_{1,1}&\leq \|x_k-x^*\|_M\|v_{x_k-x^*}\|_M^*\|\mathcal{H}(x_k)-\mathcal{H}(x^*)\|_M\\
		&\leq \gamma\left(\frac{1+\mu/\ell_{cs}^2}{1+\mu/u_{cs}^2}\right)^{1/2}\|x_k-x^*\|_M^2\\
		&=2\gamma\left(\frac{1+\mu/\ell_{cs}^2}{1+\mu/u_{cs}^2}\right)^{1/2} M_f^{\mu,g}(x_k-x^*).
	\end{align*}
	Now consider the term $E_{1,2}$ in Eq. (\ref{eq:split-E-terms}). Since the norm $\|\cdot\|_M$ is a convex function of $x$, we have by definition of convexity that $\|0\|_M-\|x_k-x^*\|_M\geq \langle v_{x_k-x^*},-(x_k-x^*)\rangle$. Therefore, we have
	\begin{align*}
		E_{1,2}=\|x_k-x^*\|_M\langle v_{x_k-x^*},x_k-x^*\rangle\geq \|x_k-x^*\|_M^2=2M_f^{\mu,g}(x_k-x^*).
	\end{align*}
	Combining the bounds on $E_{1,1}$ and $E_{1,2}$, we obtain
	\begin{align*}
		E_1=E_{1,1}-E_{1,2}\leq -2\left[1-\gamma\left(\frac{1+\mu/\ell_{cs}^2}{1+\mu/u_{cs}^2}\right)^{1/2}\right]M_f^{\mu,g}(x_k-x^*).
	\end{align*}
	We next analyze the term $E_2$ in Eq. (\ref{eq:expansion}) in the following:
	\begin{align*}
		E_2=\;&\mathbb{E}\left[\|\mathcal{H}(x_k)-x_k+w_k\|_s^2\mid\mathcal{F}_k\right]\\
		=\;&\mathbb{E}\left[\|\mathcal{H}(x_k)-\mathcal{H}(x^*)+x^*-x_k+w_k\|_s^2\mid\mathcal{F}_k\right]\tag{$\mathcal{H}(x^*)=x^*$}\\
		\leq \;&\mathbb{E}\left[\left(\|\mathcal{H}(x_k)-\mathcal{H}(x^*)\|_s+\|x_k-x^*\|_s+\|w_k\|_s\right)^2\mid\mathcal{F}_k\right]\\
		\leq \;&\mathbb{E}\left[\left(u_{cs}\|\mathcal{H}(x_k)-\mathcal{H}(x^*)\|_c+u_{cs}\|x_k-x^*\|_c+u_{es}\|w_k\|_e\right)^2\mid\mathcal{F}_k\right]\\
		\leq \;&\mathbb{E}\left[\left(2u_{cs}\|x_k-x^*\|_c+u_{es}\|w_k\|_e\right)^2\mid\mathcal{F}_k\right]\tag{Assumption \ref{as:contraction}}\\
		\leq \;&8u_{cs}^2\|x_k-x^*\|_c^2+2u_{es}^2\mathbb{E}\left[\|w_k\|_e^2\mid\mathcal{F}_k\right]\tag{$(a+b)^2\leq 2(a^2+b^2)$}\\
		\leq \;&8u_{cs}^2\|x_k-x^*\|_c^2+2u_{es}^2(A+B\|x_k\|_e^2)\tag{Assumption \ref{as:noise} (b)}\\
		\leq \;&8u_{cs}^2\|x_k-x^*\|_c^2+\frac{2u_{es}^2u_{cs}^2}{\ell_{es}^2}(A+B\|x_k\|_c^2)\tag{$u_{cs}\geq 1$ and $\ell_{es}\leq 1$}\\
		\leq \;&8u_{cs}^2\|x_k-x^*\|_c^2+\frac{2u_{es}^2u_{cs}^2}{\ell_{es}^2}(A+2B\|x_k-x^*\|_c^2+2B\|x^*\|_c^2)\\
		\leq \;&4u_{cs}^2\left(2+\frac{Bu_{es}^2}{\ell_{es}^2}\right)\|x_k-x^*\|_c^2+\frac{2u_{es}^2u_{cs}^2}{\ell_{es}^2}(A+2B\|x^*\|_c^2)\\
		\leq  \;&\frac{8u_{cs}^2u_{es}^2(B+2)}{\ell_{es}^2}f(x_k-x^*)+\frac{2u_{es}^2u_{cs}^2}{\ell_{es}^2}(A+2B\|x^*\|_c^2)\\
		\leq \;&\frac{8u_{cs}^2u_{es}^2(B+2)(\ell_{cs}^2+\mu)}{\ell_{cs}^2\ell_{es}^2}M_f^{\mu,g}(x_k-x^*)+\frac{2u_{es}^2u_{cs}^2}{\ell_{es}^2}(A+2B\|x^*\|_c^2).\tag{Lemma \ref{le:properties_moreau} (b)}
	\end{align*}
	Substituting the upper bounds we obtained for the terms $E_1$ and $E_2$ into Eq. (\ref{eq:expansion}), we finally have for all $k\geq 0$:
	\begin{align*}
		&\mathbb{E}[M_f^{\mu,g}(x_{k+1}-x^*)\mid\mathcal{F}_k]\\
		\leq\;&\left\{1-2\left[1-\gamma\left(\frac{1+\mu/\ell_{cs}^2}{1+\mu/u_{cs}^2}\right)^{1/2}\right]\epsilon_k+\frac{4u_{cs}^2u_{es}^2(B+2)L(\ell_{cs}^2+\mu)}{\mu\ell_{cs}^2\ell_{es}^2}\epsilon_k^2\right\}M_f^{\mu,g}(x_k-x^*)\\
		&+\frac{Lu_{es}^2u_{cs}^2\epsilon_k^2}{\mu\ell_{es}^2}(A+2B\|x^*\|_c^2)\\
		=\;&(1-2\alpha_2\epsilon_k+\alpha_3\epsilon_k^2) M_f^{\mu,g}(x_k-x^*)+\frac{\alpha_4(A+2B\|x^*\|_c^2)}{2(1+\mu/\ell_{cs}^2)}\epsilon_k^2,
	\end{align*}
	where the constants $\{\alpha_i\}_{1\leq i\leq 4}$ are defined in the beginning of Section \ref{subsec:recurse-contract}.
	
	\subsection{Proof of Theorem \ref{thm:sa-finite-time-bound}}\label{pf:thm:sa-finite-time-bound}
	We begin with Eq. (\ref{eq:recursion}) of Proposition \ref{prop:recursion}. When $\epsilon_0\leq \alpha_2/\alpha_3$, we have by monotonicity of $\{\epsilon_k\}$ that:
	\begin{align*}
		\mathbb{E}[M_f^{\mu,g}(x_{k+1}-x^*)\mid\mathcal{F}_k]
		\leq\;&(1-\alpha_2\epsilon_k) M_f^{\mu,g}(x_k-x^*)+\frac{\alpha_4(A+2B\|x^*\|_c^2)}{2(1+\mu/\ell_{cs}^2)}\epsilon_k^2
	\end{align*}
	for all $k\geq 0$. Taking the total expectation on both side of the previous inequality and then recursively using it, we obtain
	\begin{align*}
		\mathbb{E}[M_f^{\mu,g}(x_k-x^*)]
		\leq \prod_{j=0}^{k-1}(1-\alpha_2\epsilon_j)M_f^{\mu,g}(x_0-x^*)+\frac{\alpha_4(A+2B\|x^*\|_c^2)}{2(1+\mu/\ell_{cs}^2)}\sum_{i=0}^{k-1}\epsilon_i^2\prod_{j=i+1}^{k-1}(1-\alpha_2\epsilon_j).
	\end{align*}
	The above inequality is the finite-sample bounds of $M_f^{\mu,g}(x_k-x^*)$. To write it in terms of the original norm square $\|x_k-x^*\|_c^2$, using Lemma \ref{le:properties_moreau} (b) one more time and we finally obtain
	\begin{align*}
		\mathbb{E}\left[\|x_k-x^*\|_c^2\right]\leq\alpha_1\|x_0-x^*\|_c^2\prod_{j=0}^{k-1}(1-\alpha_2\epsilon_j)+\alpha_4(A+2B\|x^*\|_c^2)\sum_{i=0}^{k-1}\epsilon_i^2\prod_{j=i+1}^{k-1}(1-\alpha_2\epsilon_j)
	\end{align*}
	for all $k\geq 0$.

	\subsection{Proof of Corollary \ref{co:constant_step_size}}\label{pf:co_constant_step_size}
	We begin with Eq. (\ref{eq:finite_bound}) of Theorem \ref{thm:sa-finite-time-bound}. When $\epsilon_k=\epsilon\leq \alpha_2/\alpha_3$ for all $k\geq 0$, we have
	\begin{align*}
		\mathbb{E}\left[\|x_k-x^*\|_c^2\right]&\leq\alpha_1\|x_0-x^*\|_c^2\prod_{j=0}^{k-1}(1-\alpha_2\epsilon_j)+\alpha_4(A+2B\|x^*\|_c^2)\sum_{i=0}^{k-1}\epsilon_i^2\prod_{j=i+1}^{k-1}(1-\alpha_2\epsilon_j)\\
		&= \alpha_1\|x_0-x^*\|_c^2(1-\alpha_2\epsilon)^k+\alpha_4(A+2B\|x^*\|_c^2)\epsilon^2\sum_{i=0}^{k-1}(1-\alpha_2\epsilon)^{k-i-1}\\
		&\leq \alpha_1\|x_0-x^*\|_c^2(1-\alpha_2\epsilon)^k+(A+2B\|x^*\|_c^2)\frac{\alpha_4\epsilon}{\alpha_2}.
	\end{align*}
	
	\subsection{Proof of Corollary \ref{co:diminishing_step_size}}
	We begin with Eq. (\ref{eq:finite_bound})\label{pf:co_diminishing_step_size}
	\begin{align*}
		\mathbb{E}\left[\|x_k-x^*\|_c^2\right]\leq\alpha_1\|x_0-x^*\|_c^2\underbrace{\prod_{j=0}^{k-1}(1-\alpha_2\epsilon_j)}_{T_1}+\alpha_4(A+2B\|x^*\|_c^2)\underbrace{\sum_{i=0}^{k-1}\epsilon_i^2\prod_{j=i+1}^{k-1}(1-\alpha_2\epsilon_j)}_{T_2}.
	\end{align*}
	We next evaluate the terms $T_1$ and $T_2$ for different values of $\xi$ and $\epsilon$.
	
	\subsubsection{The term $T_1$}\label{subsec:T_1}
	Using the expression for $\epsilon_k$ and the relation that $e^x\geq 1+x$ for all $x\in\mathbb{R}$, we have
	\begin{align*}
		T_1=\prod_{j=0}^{k-1}(1-\alpha_2\epsilon_j)=\prod_{j=0}^{k-1}\left(1-\frac{\alpha_2\epsilon}{(j+K)^\xi}\right)\leq \exp\left(-\alpha_2\epsilon\sum_{i=0}^{k-1}\frac{1}{(j+K)^\xi}\right).
	\end{align*}
	Since the inequality $\int_{a}^{b+1}h(x)dx\leq \sum_{n=a}^{b}h(n)\leq \int_{a-1}^{b}h(x)dx$ holds for any non-increasing function $h(x)$, we have
	\begin{align*}
		T_1\leq \exp\left(-\alpha_2\epsilon\int_{0}^{k}\frac{1}{(x+K)^\xi}dx\right)
		\leq \begin{dcases}
			\left(\frac{K}{k+K}\right)^{\alpha_2\epsilon},&\xi=1,\\
			&\\
			\exp\left[-\frac{\alpha_2\epsilon}{1-\xi}\left((k+K)^{1-\xi}-K^{1-\xi}\right)\right],&\xi\in (0,1).
		\end{dcases}
	\end{align*}
	
	\subsubsection{The term $T_2$}\label{subsec:T_2}
	When $\xi=1$, using the expression of $\epsilon_k$, we have
	\begin{align*}
		T_2&=\sum_{i=0}^{k-1}\epsilon_i^2\prod_{j=i+1}^{k-1}(1-\alpha_2\epsilon_j)\\
		&=\epsilon^2 \sum_{i=0}^{k-1}\frac{1}{(i+K)^2}\prod_{j=i+1}^{k-1}\left(1-\frac{\alpha_2\epsilon}{j+K}\right)\\
		&=\epsilon^2 \sum_{i=0}^{k-1}\frac{1}{(i+K)^2}\exp\left(-\alpha_2\epsilon\sum_{j=i+1}^{k-1}\frac{1}{j+K}\right)\\
		&\leq \epsilon^2 \sum_{i=0}^{k-1}\frac{1}{(i+K)^2}\exp\left(-\alpha_2\epsilon\int_{i+1}^{k}\frac{1}{x+K}dx\right)\\
		&\leq \epsilon^2 \sum_{i=0}^{k-1}\frac{1}{(i+K)^2}\left(\frac{i+1+K}{k+K}\right)^{\alpha_2\epsilon}\\
		&\leq \frac{4\epsilon^2}{(k+K)^{\alpha_2\epsilon}}\sum_{i=0}^{k-1}\frac{1}{(i+1+K)^{2-\alpha_2\epsilon}},
	\end{align*}
	where the last line follows from
	\begin{align*}
		\left(\frac{i+1+K}{i+K}\right)^2\leq \left(\frac{K+1}{K}\right)^2\leq 4.
	\end{align*}
	We next consider the quantity $\sum_{i=0}^{k-1}\frac{1}{(i+1+K)^{2-\alpha_2\epsilon}}$, whose upper bounds depend on the relation between $\alpha_2\epsilon$ and $2$. Using the same technique as before, i.e., bounding the summation by its corresponding integral, we have the following results.
	
	\begin{enumerate}
		\item When $\epsilon\in (0,1/\alpha_2)$, we have $\sum_{i=0}^{k-1}\frac{1}{(i+1+K)^{2-\alpha_2\epsilon}}\leq \frac{1}{1-\alpha_2\epsilon}$.
		\item When $\epsilon=1/\alpha_2$, we have $\sum_{i=0}^{k-1}\frac{1}{i+1+K}\leq \log(k+K)$.
		\item When $\epsilon\in (1/\alpha_2,2/\alpha_2)$, we have $\sum_{i=0}^{k-1}\frac{1}{(i+1+K)^{2-\alpha_2\epsilon}}\leq \frac{1}{\alpha_2\epsilon-1}(k+K)^{\alpha_2\epsilon-1}$.
		\item When $\epsilon=2/\alpha_2$, we have $\sum_{i=0}^{k-1}\frac{1}{(i+1+K)^0}=k$.
		\item when $\epsilon>2/\alpha_2$, we have $\sum_{i=0}^{k-1}\frac{1}{(i+1+K)^{2-\alpha_2\epsilon}}
		\leq \frac{e}{\alpha_2\epsilon-1}(k+K)^{\alpha_2\epsilon-1}$.
	\end{enumerate}
	Combine the above five cases together and we have when $\xi=1$:
	\begin{align*}
		T_2
		\leq \frac{4\epsilon^2}{(k+K)^{\alpha_2\epsilon}}\sum_{i=0}^{k-1}\frac{1}{(i+1+K)^{2-\alpha_2\epsilon}}\leq \begin{dcases}
			\frac{4\epsilon^2}{1-\alpha_2\epsilon}\frac{1}{(k+K)^{\alpha_2\epsilon}},&\epsilon\in (0,1/\alpha_2),\\
			4\epsilon^2\frac{\log(k+K)}{k+K},&\epsilon=1/\alpha_2,\\
			\frac{4e\epsilon^2}{\alpha_2\epsilon-1}\frac{1}{k+K},&\epsilon\in (1/\alpha_2,\infty).
		\end{dcases}
	\end{align*}

	When $\xi\in (0,1)$, the approach we used earlier does not work because the integral we used to bound the sum does not admit a clean analytical expression. Here we present one way to evaluate $T_2$ based on induction. Consider the sequence $\{u_k\}_{k\geq 0}$ defined by
	\begin{align*}
		u_{0}=0,\quad u_{k+1}=(1-\alpha_2\epsilon_k)u_k+\epsilon_k^2,\quad \forall \;k\geq 0.
	\end{align*}
	It can be easily verified that $u_k=\sum_{i=0}^{k-1}\epsilon_i^2\prod_{j=i+1}^{k-1}(1-\alpha_2\epsilon_j)=T_2$. We next use induction on $u_k$ to show that when $k\geq \max(0,\left[2\xi/(\alpha_2\epsilon)\right]^{1/(1-\xi)}-K) $, we have $u_k\leq \frac{2\epsilon}{\alpha_2}\frac{1}{(k+K)^\xi}$. Since $u_{0}=0\leq \frac{2\epsilon}{\alpha_2}\frac{1}{K^\xi}$, we have the base case. Now suppose $u_k\leq \frac{2\epsilon}{\alpha_2}\frac{1}{(k+K)^\xi}$ for some $k> 0$. Consider $u_{k+1}$. We have
	\begin{align*}
		&\frac{2\epsilon}{\alpha_2}\frac{1}{(k+1+K)^\xi}-u_{k+1}\\
		=&\frac{2\epsilon}{\alpha_2}\frac{1}{(k+1+K)^\xi}-(1-\alpha_2\epsilon_k)u_k-\epsilon_k^2\\
		\geq& \frac{2\epsilon}{\alpha_2}\frac{1}{(k+1+K)^\xi}-\left(1-\frac{\alpha_2\epsilon}{(k+K)^\xi}\right)\frac{2\epsilon}{\alpha_2}\frac{1}{(k+K)^\xi}-\frac{\epsilon^2}{(k+K)^{2\xi}}\\
		=&\frac{2\epsilon}{\alpha_2}\left[\frac{1}{(k+1+K)^\xi}-\frac{1}{(k+K)^\xi}+\frac{\alpha_2\epsilon}{2}\frac{1}{(k+K)^{2\xi}}\right]\\
		=&\frac{2\epsilon}{\alpha_2}\frac{1}{(k+K)^{2\xi}}\left[\frac{\alpha_2\epsilon}{2}-(k+K)^\xi\left(1-\left(\frac{k+K}{k+1+K}\right)^{\xi}\right)\right].
	\end{align*}
	Note that
	\begin{align*}
		\left(\frac{k+K}{k+1+K}\right)^{\xi}=\left[\left(1+\frac{1}{k+K}\right)^{k+K}\right]^{-\frac{\xi}{k+K}}
		\geq \exp\left(-\frac{\xi}{k+K}\right)\geq 1-\frac{\xi}{k+K},
	\end{align*}
	where we used $(1+\frac{1}{x})^x<e$ for all $x>0$ and
	$e^x\geq 1+x$ for all $x\in\mathbb{R}$. Therefore, we obtain
	\begin{align*}
		\frac{2\epsilon}{\alpha_2}\frac{1}{(k+1+K)^\xi}-u_{k+1}&=\frac{2\epsilon}{\alpha_2}\frac{1}{(k+K)^{2\xi}}\left[\frac{\alpha_2\epsilon}{2}-(k+K)^\xi\left(1-\left(\frac{k+K}{k+1+K}\right)^{\xi}\right)\right]\\
		&\geq \frac{2\epsilon}{\alpha_2}\frac{1}{(k+K)^{2\xi}}\left[\frac{\alpha_2\epsilon}{2}-\frac{\xi}{(k+K)^{1-\xi}}\right]\\
		&\geq 0,
	\end{align*}
	where the last line follows from $K\geq \left[2\xi/(\alpha_2\epsilon)\right]^{1/(1-\xi)} $. The induction is now complete, and we have $T_2\leq \frac{2\epsilon}{\alpha_2}\frac{1}{(k+K)^{\xi}}$ for all $k\geq0$.
	
	Finally, combine the results in Subsections \ref{subsec:T_1} and \ref{subsec:T_2}, we have the finite-sample error bounds given in Corollary \ref{co:diminishing_step_size}.
	
	\subsection{Proof of Corollary \ref{co:log-dependence}}\label{ap:tune_constants}
	When $\|\cdot\|_c=\|\cdot\|_e=\|\cdot\|_\infty$, we choose the smoothing function as $g(x)=\frac{1}{2}\|x\|_p^2$ with $p\geq 2$. It is known that $g(x)$ is $(p-1)$ -- smooth w.r.t. $\|\cdot\|_p$ (Example 5.11 \cite{beck2017first}). Moreover, we have in this case $\ell_{cs}=\ell_{es}=1$ and $u_{cs}=u_{es}=d^{1/p}$. It follows that
	\begin{align}\label{constants1}
		\alpha_1&=\frac{1+\mu}{1+\mu/d^{2/p}}, &\alpha_2&=1-\gamma\left(\frac{1+\mu}{1+\mu/d^{2/p}}\right)^{1/2},\nonumber\\ \alpha_3&=4d^{4/p}(p-1)(1+1/\mu)(B+2), & \alpha_4&=2d^{4/p}(p-1)(1+1/\mu).
	\end{align}
	Note that $\alpha_3$ and $\alpha_4$ are both proportional to $d^{4/p}(p-1)$. Let $h(p)=d^{4/p}(p-1)$. Assume without loss of generality that $d\geq 2$, then we have $\min_{p\geq 2} h(p)\leq h(4\log(d))\leq 4e\log(d)$. Hence, with $p=4\log(d)$, the dimension dependence of $\alpha_3$ and $\alpha_4$ is $\log(d)$.
	
	Now for the choice of $\mu$, observe that $\alpha_1$ and $\alpha_2$ are in favor of small $\mu$, while $\alpha_3$ and $\alpha_3$ are in favor of large $\mu$. To balance the effect, we choose $\mu=(1/2+1/2\gamma)^2-1$. Note that this choice of $\mu$ gives
	\begin{align*}
		1+\frac{1}{\mu}=\frac{(1+\gamma)^2}{(\gamma+1)^2-4\gamma^2}=\frac{(1+\gamma)^2}{(1-\gamma)(1+3\gamma)}\leq \frac{1+\gamma}{1-\gamma}\leq \frac{2}{1-\gamma}.
	\end{align*}
	Substituting $p=4\log(d)$ and $\mu=(1/2+1/2\gamma)^2-1$ into Eq. (\ref{constants1}), we obtain
	\begin{align*}
		\alpha_1&=\frac{1+\mu}{1+\mu/d^{2/p}}=\sqrt{e}\frac{1+\mu}{\sqrt{e}+\mu}\leq \sqrt{e}\leq \frac{3}{2},\\ \alpha_2&=1-\gamma(\frac{1+\mu}{1+\mu/d^{2/p}})^{1/2}\geq 1-\gamma(1+\mu)^{1/2}= \frac{1-\gamma}{2},\\
		\alpha_3&=4(1+1/\mu)(B+2)d^{4/p}(p-1)\leq \frac{32e(B+2)\log(d)}{1-\gamma}, \\ \alpha_4&=2(1+1/\mu)d^{4/p}(p-1)\leq \frac{16e\log(d)}{1-\gamma}.
	\end{align*}
	
	\subsection{Proof of Theorem \ref{thm:average}}\label{ap:pf:average}
	Note that for any $\epsilon\in (0,1)$ and $x,y\in\mathbb{R}^d$, we have
	\begin{align*}
		\|(1-\epsilon)x+\epsilon y\|_2^2=(1-\epsilon)\|x\|_2^2+\epsilon\|y\|_2^2-\epsilon(1-\epsilon)\|x-y\|_2^2.
	\end{align*}
	It follows that for any $x^*\in\mathcal{X}$:
	\begin{align}
		&\mathbb{E}[\|x_{k+1}-x^*\|_2^2\mid\mathcal{F}_k]\nonumber\\
		=\;&\mathbb{E}[\|(1-\epsilon_k)(x_k-x^*)+\epsilon_k(\mathcal{H}(x_k)-x^*+w_k)\|_2^2\mid\mathcal{F}_k]\nonumber\\
		=\;&(1-\epsilon_k)\|x_k-x^*\|_2^2+\epsilon_k\mathbb{E}[\|\mathcal{H}(x_k)-x^*+w_k\|_2^2\mid\mathcal{F}_k]-(1-\epsilon_k)\epsilon_k\mathbb{E}[\|\mathcal{H}(x_k)-x_k+w_k\|_2^2\mid\mathcal{F}_k]\nonumber\\
		\leq \;&(1-\epsilon_k)\|x_k-x^*\|_2^2+\epsilon_k\|x_k-x^*\|_2^2+\epsilon_k\mathbb{E}[\|w_k\|_2^2\mid\mathcal{F}_k]-(1-\epsilon_k)\epsilon_k(\|\mathcal{H}(x_k)-x_k\|_2^2+\mathbb{E}[\|w_k\|_2^2\mid\mathcal{F}_k])\nonumber\\
		=\;&\|x_k-x^*\|_2^2-(1-\epsilon_k)\epsilon_k(\|\mathcal{H}(x_k)-x_k\|_2^2)+\epsilon_k^2\mathbb{E}[\|w_k\|_2^2\mid\mathcal{F}_k]\nonumber\\
		\leq \;&\|x_k-x^*\|_2^2-(1-\epsilon_k)\epsilon_k(\|\mathcal{H}(x_k)-x_k\|_2^2)+A\epsilon_k^2.\label{eq:1}
	\end{align}
	\paragraph{Almost Sure Convergence:}
	The previous inequality implies for any $k\geq 0$:
	\begin{align*}
		\mathbb{E}[\inf_{x^*\in\mathcal{X}}\|x_{k+1}-x^*\|_2^2\mid\mathcal{F}_k]&\leq \inf_{x^*\in\mathcal{X}}\mathbb{E}[\|x_{k+1}-x^*\|_2^2\mid\mathcal{F}_k]\\
		&\leq \inf_{x^*\in\mathcal{X}}\|x_k-x^*\|_2^2-(1-\epsilon_k)\epsilon_k\|\mathcal{H}(x_k)-x_k\|_2^2+A\epsilon_k^2.
	\end{align*}
	Now using the Supermartingale Convergence Theorem (Proposition 4.2 in \cite{bertsekas1996neuro}) and the definition of dist$(x,\mathcal{X})$, we have 
	\begin{enumerate}
		\item $\text{dist}(x_k,\mathcal{X})$ converges almost surely to a non-negative random variable $X$.
		\item $\sum_{k=0}^{\infty}(1-\epsilon_k)\epsilon_k\|\mathcal{H}(x_k)-x_k\|_2^2<\infty$ almost surely.
	\end{enumerate} 
	It remains to show that $X$ is almost surely zero. Consider a sample path $\{x_k\}$ and suppose $\lim_{k\rightarrow\infty}\text{dist}(x_k,\mathcal{X})>0$. Then there exists $\delta>0$ and $K_1>0$ such that $\text{dist}(x_k,\mathcal{X})\geq \delta$ for all $k\geq K_1$. Due to convergence, the sequence $\{\text{dist}(x_k,\mathcal{X})\}$ is also bounded above. Let $M$ be an upper bound. Define $E=\{x\mid \delta\leq \text{dist}(x,\mathcal{X})\leq M \}$. Since $\text{dist}(x,\mathcal{X})$ as a function of $x$ is continuous, and the inverse image of a compact set under a continuous mapping is also compact, the set $E$ is compact. By continuity of $\|\mathcal{H}(x)-x\|_2^2$ and Weierstrass theorem, there exists $\delta_1>0$ such that $\|\mathcal{H}(x)-x\|_2^2\geq \delta_1$ for all $x\in E$. It follows that for the sample path $\{x_k\}$, we have 
	\begin{align*}
		\sum_{k=0}^{\infty}(1-\epsilon_k)\epsilon_k\|\mathcal{H}(x_k)-x_k\|_2^2\geq \delta_1\sum_{k=K_1}^{\infty}(1-\epsilon_k)\epsilon_k=\infty.
	\end{align*}
	Since we already know that $\sum_{k=0}^{\infty}(1-\epsilon_k)\epsilon_k\|\mathcal{H}(x_k)-x_k\|_2^2<\infty$ almost surely, we must have $\lim_{k\rightarrow\infty}\text{dist}(x_k,\mathcal{X})=0$ almost surely.
	
	\paragraph{Finite-Sample Bounds:}
	Taking the total expectation on both sides of Eq. (\ref{eq:1}) then telescoping, we have for all $k\geq 0$:
	\begin{align*}
		\mathbb{E}[\|x_{k+1}-x^*\|_2^2]\leq \|x_0-x^*\|_2^2-\sum_{i=0}^{k}(1-\epsilon_i)\epsilon_i\mathbb{E}[\|\mathcal{H}(x_i)-x_i\|_2^2]+A\sum_{i=0}^{k}\epsilon_i^2.
	\end{align*}
	It follows that for all $k\geq 0$:
	\begin{align*}
		\min_{0\leq i\leq k}\mathbb{E}[\|\mathcal{H}(x_i)-x_i\|_2^2]\sum_{i=0}^{k}(1-\epsilon_i)\epsilon_i
		\leq \sum_{i=0}^{k}(1-\epsilon_i)\epsilon_i\mathbb{E}[\|\mathcal{H}(x_i)-x_i\|_2^2]\leq \|x_0-x^*\|_2^2+A\sum_{i=0}^{k}\epsilon_i^2,
	\end{align*}
	which implies
	\begin{align*}
		\min_{0\leq i\leq k}\mathbb{E}[\|\mathcal{H}(x_i)-x_i\|_2^2]\leq  \frac{\|x_0-x^*\|_2^2+A\sum_{i=0}^{k}\epsilon_i^2}{\sum_{i=0}^{k}(1-\epsilon_i)\epsilon_i}\leq \frac{D^2+A\sum_{i=0}^{k}\epsilon_i^2}{\sum_{i=0}^{k}(1-\epsilon_i)\epsilon_i}.
	\end{align*}
	Therefore, when using constant stepsize $\epsilon\in (0,1)$, we have
	\begin{align*}
		\min_{0\leq i\leq k}\mathbb{E}[\|\mathcal{H}(x_i)-x_i\|_2^2]\leq \frac{D^2+A(k+1)\epsilon^2}{(k+1)(1-\epsilon)\epsilon}=\frac{D^2}{(k+1)(1-\epsilon)\epsilon}+\frac{A\epsilon}{1-\epsilon}.
	\end{align*}
	When using diminishing stepsizes $\epsilon_k=\frac{\epsilon}{\sqrt{k+1}}$, since
	\begin{align*}
	    \sum_{i=0}^k(1-\epsilon_i)\epsilon_i\geq (1-\epsilon)\epsilon\sum_{i=0}^k\frac{1}{\sqrt{i+1}}\geq (1-\epsilon)\epsilon\int_{0}^k\frac{1}{\sqrt{x+1}}dx=2(1-\epsilon)\epsilon((k+1)^{1/2}-1),
	\end{align*}
	and
	\begin{align*}
	    \sum_{i=0}^k\epsilon_i^2=\epsilon^2\sum_{i=1}^{k+1}\frac{1}{i}\leq \epsilon^2\left(1+\int_{1}^k\frac{1}{x}dx\right)=\epsilon^2(1+\log(k)),
	\end{align*}
	we have
	\begin{align*}
		\min_{0\leq i\leq k}\mathbb{E}[\|\mathcal{H}(x_i)-x_i\|_2^2]\leq\frac{D^2+A\epsilon^2(1+\log(k))}{2(1-\epsilon)\epsilon((k+1)^{1/2}-1)}.
	\end{align*}
	
	When using diminishing stepsizes $\epsilon_k=\frac{\epsilon}{k+1}$, since 
	\begin{align*}
		\sum_{i=0}^{k}(1-\epsilon_i)\epsilon_i\geq (1-\epsilon)\epsilon\sum_{i=0}^{k}\frac{1}{i+1}\geq (1-\epsilon)\epsilon \int_{0}^{k}\frac{1}{x+1}\;dx=(1-\epsilon)\epsilon\log(k+1),
	\end{align*}
	and 
	\begin{align*}
		\sum_{i=0}^{k}\epsilon_i^2=\epsilon^2\sum_{i=0}^{k}\frac{1}{(i+1)^2}\leq \epsilon^2+\epsilon^2\sum_{i=1}^{k}\left(\frac{1}{i}-\frac{1}{i+1}\right)\leq 2\epsilon^2,
	\end{align*}
	we have
	\begin{align*}
		\min_{0\leq i\leq k}\mathbb{E}[\|\mathcal{H}(x_i)-x_i\|_2^2]\leq \frac{D^2+2A\epsilon^2}{(1-\epsilon)\epsilon}\frac{1}{\log(k+1)}.
	\end{align*}
	
	\section{Proofs of All Technical Results in Section \ref{sec:applications}}
	
	\subsection{Proof of Proposition \ref{prop:vtrace}}\label{appendix-vtrace}
	The ideas for the proof of Proposition \ref{prop:vtrace} is similar to that in \cite{espeholt2018impala}, and we include it here for completeness. For simplicity, we use $c_{n,m}^s$ for $\prod_{j=n}^{m}c_j^s$.
	
	\begin{enumerate}
		\item For any integer $n\geq 1$, define the operator $\mathcal{T}_n:\mathbb{R}^{|\mathcal{S}|}\mapsto\mathbb{R}^{|\mathcal{S}|}$ by
		\begin{align*}
			[\mathcal{T}_n(V)](s)=\mathbb{E}_{\pi'}\left[\sum_{t=0}^{n-1}\beta^tc_{0,t-1}\rho_t\left(\mathcal{R}(S_t,A_t)+\beta V(S_{t+1})-V(S_t)\right)\;\middle|\; S_0=s\right]+V(s),\quad \forall\;s\in\mathcal{S},
		\end{align*}
		and the random process $\{w_k\}$ by
		\begin{align*}
			w_k(s)=\sum_{t=0}^{n-1}\beta^tc_{0,t-1}^s\rho_t^s (\mathcal{R}(S_t^s,A_t^s)+\beta V_k(S_{t+1}^s)-V_k(S_t^s))+V_k(s)-[\mathcal{T}_n(V_k)](s),\quad \forall\;s\in\mathcal{S}.
		\end{align*}
		Then the V-trace algorithm (\ref{vtrace_algorithm}) can be written as $	V_{k+1}=V_k+\epsilon_k\left(\mathcal{T}_n(V_k)-V_k+w_k\right)$.
		\item Consider $\kappa_c$ and $\kappa_\rho$. On the one hand, we have
		\begin{align*}
			\kappa_\rho&=\min_{s\in\mathcal{S}}\mathbb{E}_{\pi'}\left[\min\left(\bar{\rho},\frac{\pi(A_t\mid S_t)}{\pi'(A_t\mid S_t)}\right)\;\middle|\;S_t=s\right]\\
			&\\
			&\geq \min_{s\in\mathcal{S}}\mathbb{E}_{\pi'}\left[\min\left(\bar{c},\frac{\pi(A_t\mid S_t)}{\pi'(A_t\mid S_t)}\right)\;\middle|\;S_t=s\right]\tag{$\bar{c}\leq \bar{\rho}$}\\
			&=\kappa_c\\
			&=\min_{s\in\mathcal{S}}\sum_{a:\pi'(a|s)>0}\min\left(\bar{c}\pi'(a|s),\pi(a|s)\right)\\
			&\geq \min_{s\in\mathcal{S}}\sum_{a:\pi(a|s)>0}\min\left(\bar{c}\pi'(a|s),\pi(a|s)\right)\tag{Assumption \ref{as:coverage}}\\
			&>0,
		\end{align*}
		where the last line follows from $\mathcal{S}$ being finite. On the other hand, we also have 
		\begin{align*}
			\kappa_\rho&=\min_{s\in\mathcal{S}}\mathbb{E}_{\pi'}\left[\min\left(\bar{\rho},\frac{\pi(A_t\mid S_t)}{\pi'(A_t\mid S_t)}\right)\;\middle|\;S_t=s\right]\\
			&\leq \min_{s\in\mathcal{S}}\sum_{a:\pi'(a|s)>0}\min\left(\bar{\rho}\pi'(a|s),\pi(a|s)\right)\\
			&\leq \min_{s\in\mathcal{S}}\sum_{a\in\mathcal{A}}\pi(a|s)\\
			&=1.
		\end{align*}
		Combine the previous two results and we have $0<\kappa_c\leq \kappa_\rho\leq 1$. 
		
		We next show the contraction property of the operator $\mathcal{T}_n$. We begin by rewriting $\mathcal{T}_n$ in the following way:
		\begin{align*}
			[\mathcal{T}_n(V)](s)=\mathbb{E}_{\pi'}\bigg[&\sum_{t=0}^{n}\beta^{t-1}c_{0,t-2}(\rho_{t-1}\mathcal{R}(S_{t-1},A_{t-1})+\beta(\rho_{t-1}-c_{t-1}\rho_t)V(S_t))\\
			&\quad +\beta^{n}c_{0,n-1}\rho_{n}V(S_{n})\;\bigg|\;S_0=s\bigg].
		\end{align*}
		For any $V_1,V_2:\mathbb{R}^{|\mathcal{S}|}\mapsto\mathbb{R}$ and $s\in\mathcal{S}$, we have
		\begin{align}
			&[\mathcal{T}_n(V_1)](s)-[\mathcal{T}_n(V_2)](s)\nonumber\\
			=\;&\mathbb{E}_{\pi'}\left[\sum_{t=0}^{n}\beta^{t}c_{0,t-2}(\rho_{t-1}-c_{t-1}\rho_t)(V_1(S_t)-V_2(S_t))\;\bigg|\;S_0=s\right]\nonumber\\
			&+\mathbb{E}_{\pi'}\left[\beta^{n}c_{0,n-1}\rho_{n}(V_1(S_{n})-V_2(S_{n}))\;\bigg|\;S_0=s\right]\nonumber\\
			\leq \;&\sum_{t=0}^{n}\beta^{t}\mathbb{E}_{\pi'}\left[c_{0,t-2}(\rho_{t-1}-c_{t-1}\rho_t)(V_1(S_t)-V_2(S_t))\;\bigg|\;S_0=s\right]\label{ine:107}\\
			&+\beta^{n}\mathbb{E}_{\pi'}\left[c_{0,n-1}\rho_{n}\mid S_0=s\right]\|V_1-V_2\|_\infty.\nonumber
		\end{align}
		Since $\bar{\rho}\geq \bar{c}$, we have $\rho_t\geq c_t$ for all $t$. Therefore, using the Markov property and we have
		\begin{align}
			\mathbb{E}_{\pi'}\left[\rho_{t-1}-c_{t-1}\rho_t\mid \mathcal{F}_t\right]&\geq c_{t-1}\left(1- \mathbb{E}_{\pi'}\left[\rho_t\mid \mathcal{F}_t\right]\right)\nonumber\\
			&\geq c_{t-1}\left(1-\sum_{a:\pi'(a\mid s)>0}\pi'(a \mid S_t)\frac{\pi(a \mid S_t)}{\pi'(a \mid S_t)}\right)\nonumber\\
			&=0,\label{ine:positive}
		\end{align}
		where $\mathcal{F}_t$ denotes the $\sigma$-algebra generated by $\{S_0,A_0,...,S_{t-1}, A_{t-1},S_t\}$.
		It follows that
		\begin{align*}
			&\sum_{t=0}^{n}\beta^{t}\mathbb{E}_{\pi'}\left[c_{0,t-2}(\rho_{t-1}-c_{t-1}\rho_t)(V_1(S_t)-V_2(S_t))\;\bigg|\;S_0=s\right]\\
			=\;&\sum_{t=0}^{n}\beta^{t}\mathbb{E}_{\pi'}\left[\mathbb{E}_{\pi'}\left[c_{0,t-2}(\rho_{t-1}-c_{t-1}\rho_t)(V_1(S_t)-V_2(S_t))\;\bigg|\mathcal{F}_t\right]\;\bigg|\;S_0=s\right]\\
			=\;&\sum_{t=0}^{n}\mathbb{E}_{\pi'}\left[\mathbb{E}_{\pi'}\left[c_{0,t-2}(\rho_{t-1}-c_{t-1}\rho_t)\;\bigg|\mathcal{F}_t\right](V_1(S_t)-V_2(S_t))\;\bigg|\;S_0=s\right]\\
			\leq \;&\sum_{t=0}^{n}\beta^{t}\mathbb{E}_{\pi'}\left[\mathbb{E}_{\pi'}\left[c_{0,t-2}(\rho_{t-1}-c_{t-1}\rho_t)\;\bigg|\mathcal{F}_t\right]\;\bigg|\;S_0=s\right]\|V_1-V_2\|_\infty\\
			=\;&\sum_{t=0}^{n}\beta^{t}\mathbb{E}_{\pi'}\left[c_{0,t-2}(\rho_{t-1}-c_{t-1}\rho_t)\;\bigg|S_0=s\right]\|V_1-V_2\|_\infty.
		\end{align*}
		Using the previous result in Eq. (\ref{ine:107}) and we have
		\begin{align*}
			[\mathcal{T}_n(V_1)](s)-[\mathcal{T}_n(V_2)](s)
			\leq \;&\sum_{t=0}^{n}\beta^{t}\mathbb{E}_{\pi'}\left[c_{0,t-2}(\rho_{t-1}-c_{t-1}\rho_t)(V_1(S_t)-V_2(S_t))\;\bigg|\;S_0=s\right]\\
			&+\beta^{n}\mathbb{E}_{\pi'}\left[c_{0,n-1}\rho_{n}\mid S_0=s\right]\|V_1-V_2\|_\infty\\
			\leq \;&\bigg\{\sum_{t=0}^{n}\beta^t\mathbb{E}_{\pi'}\left[c_{0,t-2}(\rho_{t-1}-c_{t-1}\rho_t)\;\bigg|S_0=s\right]\\
			&+\beta^{n}\mathbb{E}_{\pi'}\left[c_{0,n-1}\rho_{n}\mid S_0=s\right]\bigg\}\|V_1-V_2\|_\infty.
		\end{align*}
		Switching the role between $V_1$ and $V_2$, then we have by symmetry that
		\begin{align*}
			&|[\mathcal{T}_nV_1](s)-[\mathcal{T}_nV_2](s)|\\
			\leq \;&\left\{\sum_{t=0}^{n}\beta^t\mathbb{E}_{\pi'}\left[c_{0,t-2}(\rho_{t-1}-c_{t-1}\rho_t)\mid S_0=s\right]+\beta^{n}\mathbb{E}_{\pi'}\left[c_{0,n-1}\rho_{n}\mid S_0=s\right]\right\}\|V_1-V_2\|_\infty\\
			=\;&\left(1-(1-\beta)\sum_{t=0}^{n-1}\beta^{t}\mathbb{E}_{\pi'}\left[c_{0,t-1}\rho_{t}\mid S_0=s\right]\right)\|V_1-V_2\|_\infty.
		\end{align*}
		which implies
		\begin{align*}
			\|\mathcal{T}_n(V_1)-\mathcal{T}_n(V_2)\|_\infty\leq \left(1-(1-\beta)\sum_{t=0}^{n-1}\beta^{t}\min_{s\in\mathcal{S}}\mathbb{E}_{\pi'}\left[c_{0,t-1}\rho_{t}\mid S_0=s\right]\right)\|V_1-V_2\|_\infty.
		\end{align*}
		Now consider $\min_{s\in\mathcal{S}}\mathbb{E}_{\pi'}\left[c_{0,t-1}\rho_{t}\mid S_0=s\right]$ for any $t\in [0,n-1]$. Using the definition of $\kappa_c$ and $\kappa_\rho$, we have
		\begin{align*}
			\min_{s\in\mathcal{S}}\mathbb{E}_{\pi'}\left[c_{0,t-1}\rho_{t}\mid S_0=s\right]&=\min_{s\in\mathcal{S}}\mathbb{E}_{\pi'}\left[\mathbb{E}\left[c_{0,t-1}\rho_{t}\mid \mathcal{F}_t\right]\mid S_0=s\right]\\
			&=\min_{s\in\mathcal{S}}\mathbb{E}_{\pi'}\left[c_{0,t-1}\mathbb{E}\left[\rho_{t}\mid \mathcal{F}_t\right]\mid S_0=s\right]\\
			&\geq \min_{s\in\mathcal{S}}\kappa_\rho\mathbb{E}_{\pi'}\left[c_{0,t-1}\mid S_0=s\right]\\
			&\geq \cdots\\
			&\geq \kappa_\rho\kappa_c^t.
		\end{align*}
		It follows that 
		\begin{align*}
			\left(1-(1-\beta)\sum_{t=0}^{n-1}\beta^{t}\min_{s\in\mathcal{S}}\mathbb{E}_{\pi'}\left[c_{0,t-1}\rho_{t}\mid S_0=s\right]\right)&\leq 1-(1-\beta)\kappa_\rho\sum_{t=0}^{n-1}(\beta\kappa_c)^{t}\\
			&=1-\frac{(1-\beta)(1-(\beta\kappa_c)^{n})\kappa_\rho}{1-\beta\kappa_c}.
		\end{align*}
		Hence the operator $\mathcal{T}_n$ is a contraction w.r.t. $\|\cdot\|_\infty$, with contraction factor $\gamma= 1-\frac{(1-\beta)(1-(\beta\kappa_c)^{n})\kappa_\rho}{1-\beta\kappa_c}$.
		\item It is enough to show that $V_{\pi_{\bar{\rho}}}$ is a fixed-point of $\mathcal{T}_n$, the uniqueness part follows from the Banach fixed-point theorem \cite{debnath2005introduction}. For any $t\in [0,n-1]$, we have
		\begin{align*}
			&\mathbb{E}_{\pi'}\left[\rho_t\left(\mathcal{R}(S_t,A_t)+\beta V_{\pi_{\bar{\rho}}}(S_{t+1})-V_{\pi_{\bar{\rho}}}(S_t)\right)\mid S_t\right]\\
			=\;&\sum_{a:\pi'(a \mid S_t)>0}\pi'(a|S_t)\min\left(\bar{\rho},\frac{\pi(a|S_t)}{\pi'(a|S_t)}\right)\left(\mathcal{R}(S_t,a)+\beta \sum_{s'\in\mathcal{S}}P_a(S_t,s')V_{\pi_{\bar{\rho}}}(s')-V_{\pi_{\bar{\rho}}}(S_t)\right)\\
			=\;&\sum_{a:\pi'(a \mid S_t)>0}\pi_{\bar{\rho}}(a|S_t)\!\left[\mathcal{R}(S_t,a)\!+\!\beta \sum_{s'\in\mathcal{S}}P_a(S_t,s')V_{\pi_{\bar{\rho}}}(s')\!-\! V_{\pi_{\bar{\rho}}}(S_t)\right]\sum_{a:\pi'(a \mid S_t)>0}\min(\bar{\rho}\pi'(a \mid S_t),\pi(a \mid S_t))\\
			=\;&0,
		\end{align*}
		where the last line follows from the Bellman's equation for $V_{\pi_{\bar{\rho}}}$. Therefore, using the tower property of the conditional expectation and the Markov property, we have $\mathcal{T}_n(V_{\pi_{\bar{\rho}}})=V_{\pi_{\bar{\rho}}}$, hence $V_{\pi_{\bar{\rho}}}$ is a fixed-point of the operator $\mathcal{T}_n$.
		
		We next analyze the difference between $V_{\pi_{\bar{\rho}}}$ and $V_\pi$ in terms of $\pi_{\bar{\rho}}$ and $\pi$. We first show in the following lemma that the value function as a function of its corresponding policy is Lipschitz continuous.
		\begin{lemma}\label{policy-Lipschitz}
			For any two policies $\pi_1$ and $\pi_2$, their corresponding value functions $V_{\pi_1}$ and $V_{\pi_2}$ satisfy $\|V_{\pi_1}-V_{\pi_2}\|_\infty\leq \frac{2}{(1-\beta)^2}\|\pi_1-\pi_2\|_\infty$.
		\end{lemma}
		\begin{proof}[Proof of Lemma \ref{policy-Lipschitz}]
		Note that for any policy $\pi$, its corresponding value function $V_\pi$ satisfies the following Bellman's equation:
		\begin{align}\label{eq:Bellman}
			V_\pi=R_\pi+\beta P_\pi V_\pi,
		\end{align}
		where $R_\pi(s)=\sum_{a\in\mathcal{A}}\pi(a|s)\mathcal{R}(s,a)$ for all $s\in\mathcal{S}$, and $P_\pi(s,s')=\sum_{a\in\mathcal{A}}\pi(a|s)P_a(s,s')$ for any $s,s'\in\mathcal{S}$. Using Eq. (\ref{eq:Bellman}) for $V_{\pi_1}$ and $V_{\pi_2}$, and we have
		\begin{align*}
			R_{\pi_1}-R_{\pi_2}&=(I-\beta P_{\pi_1})V_{\pi_1}-(I-\beta P_{\pi_2})V_{\pi_2}\\
			&=(I-\beta P_{\pi_1})V_{\pi_1}-(I-\beta P_{\pi_1})V_{\pi_2}+(I-\beta P_{\pi_1})V_{\pi_2}-(I-\beta P_{\pi_2})V_{\pi_2}\\
			&=(I-\beta P_{\pi_1})(V_{\pi_1}-V_{\pi_2})-\beta (P_{\pi_1}-P_{\pi_2})V_{\pi_2}.
		\end{align*}
		Since the matrix $I-\beta P_\pi$ is invertible for any policy $\pi$ \cite{bertsekas1995dynamic}, we have
		\begin{align}
			\|V_{\pi_1}-V_{\pi_2}\|_\infty&=\|(I-\beta P_{\pi_1})^{-1}[(R_{\pi_1}-R_{\pi_2})+\beta (P_{\pi_1}-P_{\pi_2})V_{\pi_2}]\|_\infty\nonumber\\
			&\leq \|(I-\beta P_{\pi_1})^{-1}\|_\infty\left[\|R_{\pi_1}-R_{\pi_2}\|_\infty+\beta \|P_{\pi_1}-P_{\pi_2}\|_\infty \|V_{\pi_2}\|_\infty\right]\label{eq:V_lipschitz}.
		\end{align}
		We next control all the terms on the r.h.s. of the preceding inequality. For the term $\|(I-\beta P_{\pi_1})^{-1}\|_\infty$, we have by definition of the matrix sup-norm that
		\begin{align*}
			\|(I-\beta P_{\pi_1})^{-1}\|_\infty^{-1}
			&=\inf_{\|x\|_\infty=1}\|(I-\beta P_{\pi_1})x\|_\infty\\
			&\geq \inf_{\|x\|_\infty=1}\|x\|_\infty-\beta\|P_{\pi_1} x\|_\infty\\
			&=1-\beta\sup_{\|x\|_\infty=1}\|P_{\pi_1} x\|_\infty\\
			&= 1-\beta.
		\end{align*}
		It follows that $\|(I-\beta P_{\pi_1})^{-1}\|_\infty\leq \frac{1}{1-\beta}$. Next, for the term $\|R_{\pi_1}-R_{\pi_2}\|_\infty$, we have
		\begin{align*}
			\|R_{\pi_1}-R_{\pi_2}\|_\infty&=\max_{s\in\mathcal{S}}|R_{\pi_1}(s)-R_{\pi_2}(s)|\\
			&=\max_{s\in\mathcal{S}}|\sum_{a\in\mathcal{A}}(\pi_1(a|s)-\pi_2(a|s)\mathcal{R}(s,a)|\\
			&\leq \max_{s\in\mathcal{S}}\sum_{a\in\mathcal{A}}|\pi_1(a|s)-\pi_2(a|s)|\tag{$\mathcal{R}(s,a)\in [0,1]$}\\
			&=\|\pi_1-\pi_2\|_\infty.
		\end{align*}
		
		Finally we consider the term $\|P_{\pi_1}-P_{\pi_2}\|_\infty \|V_{\pi_2}\|_\infty$. It is clear that $\|V_{\pi_2}\|_\infty\leq \sum_{k=0}^{\infty}\beta^k= \frac{1}{1-\beta}$. For $\|P_{\pi_1}-P_{\pi_2}\|_\infty$, we have
		\begin{align*}
			\|P_{\pi_1}-P_{\pi_2}\|_\infty&=\max_{s\in\mathcal{S}}\sum_{s'\in\mathcal{S}}|P_{\pi_1}(s,s')-P_{\pi_2}(s,s')|\\
			& =\max_{s\in\mathcal{S}}\sum_{s'\in\mathcal{S}}|\sum_{a\in\mathcal{A}}(\pi_1(a|s)-\pi_2(a|s))P_a(s,s')|\\
			& \leq \max_{s\in\mathcal{S}}\sum_{s'\in\mathcal{S}}\sum_{a\in\mathcal{A}}|(\pi_1(a|s)-\pi_2(a|s))P_a(s,s')|\\
			& = \max_{s\in\mathcal{S}}\sum_{a\in\mathcal{A}}\sum_{s'\in\mathcal{S}}|(\pi_1(a|s)-\pi_2(a|s))P_a(s,s')|\\
			& =\max_{s\in\mathcal{S}}\sum_{a\in\mathcal{A}}|\pi_1(a|s)-\pi_2(a|s)|\\
			&=\|\pi_1-\pi_2\|_\infty.
		\end{align*}
		
		Using the upper bounds we obtained for the terms on the r.h.s. of Eq. (\ref{eq:V_lipschitz}), we have
		\begin{align*}
			\|V_{\pi_1}-V_{\pi_2}\|_\infty
			&\leq \|(I-\beta P_{\pi_1})^{-1}\|_\infty\left[\|R_{\pi_1}-R_{\pi_2}\|_\infty+\beta \|P_{\pi_1}-P_{\pi_2}\|_\infty \|V_{\pi_2}\|_\infty\right]\\
			&\leq \frac{1}{1-\beta}\left[\|\pi_1-\pi_2\|_\infty+\frac{1}{1-\beta}\|\pi_1-\pi_2\|_\infty\right]\\
			&\leq \frac{2}{(1-\beta)^2}\|\pi_1-\pi_2\|_\infty.
		\end{align*}
		\end{proof}

		Using Lemma \ref{policy-Lipschitz} on $V_{\pi_{\bar{\rho}}}$ and $V_\pi$, we have
		\begin{align*}
			\|V_{\pi_{\bar{\rho}}}-V_\pi\|_\infty&\leq \frac{2}{(1-\beta)^2}\|\pi_{\bar{\rho}}-\pi\|_\infty.
		\end{align*}
		\item By definition of $\{w_k\}$, we have $\mathbb{E}[w_k\mid\mathcal{F}_k]=0$. Moreover, we have for all $s\in\mathcal{S}$:
		\begin{align*}
			|w_k(s)|&=\left|\sum_{t=0}^{n-1}\beta^tc_{0,t-1}^s\rho_t^s \left(\mathcal{R}(S_t,A_t)+\beta V_k(S_{t+1}^s)-V_k(S_t^s)\right)+V_k(s)-[\mathcal{T}_n(V_k)](s)\right|\\
			&\leq 2\sum_{t=0}^{n-1}(\beta\bar{c})^t\bar{\rho} \left(1+(\beta+1)\|V_k\|_\infty\right)\\
			&\leq 4\bar{\rho}(1+\|V_k\|_\infty)\sum_{t=0}^{n-1}(\beta\bar{c})^t\\
			&\leq \begin{dcases}
				4\bar{\rho}(1+\|V_k\|_\infty)\frac{1}{1-\beta\bar{c}},&\text{ when }\beta\bar{c}<1,\\
				4\bar{\rho}(1+\|V_k\|_\infty)n,&\text{ when }\beta\bar{c}=1,\\
				4\bar{\rho}(1+\|V_k\|_\infty)\frac{(\beta\bar{c})^{n}}{\beta\bar{c}-1},&\text{ when }\beta\bar{c}>1.
			\end{dcases}
		\end{align*}
		Therefore, we have $\mathbb{E}_{\pi'}\left[\|w_k\|_\infty^2\mid\mathcal{F}_k\right]\leq A(1+\|V_k\|_\infty^2)$, where
		\begin{align*}
			A=\begin{dcases}
				\frac{32\bar{\rho}^2}{(1-\beta\bar{c})^2},&\text{when }\beta\bar{c}<1,\\
				32\bar{\rho}^2n^2,&\text{when }\beta\bar{c}=1,\\
				\frac{32\bar{\rho}^2(\beta\bar{c})^{2n}}{(\beta\bar{c}-1)^2},&\text{when }\beta\bar{c}>1.
			\end{dcases}
		\end{align*}
	\end{enumerate}

	\subsection{Proof of Theorem \ref{thm:vtrace}}\label{pf:thm:vtrace}
	Since we have in this case $\|\cdot\|_c=\|\cdot\|_e=\|\cdot\|_\infty$, Corollary \ref{co:log-dependence} is applicable. Let $g(x)=\frac{1}{2}\|x\|_p^2$ with $p=4\log |\mathcal{S}|$, and let $\mu=(1/2+1/(2\gamma))^2-1$. Then we have
	\begin{align*}
		\alpha_1&\leq \frac{3}{2}:=\bar{\alpha}_1,&\quad
		\alpha_2&\geq \frac{1-\gamma}{2}:=\bar{\alpha}_2\\
		\alpha_3&\leq  \frac{32e(A+2)\log |\mathcal{S}|}{1-\gamma}:=\bar{\alpha}_3,&\quad
		\alpha_4&\leq \frac{16e\log |\mathcal{S}|}{1-\gamma}:=\bar{\alpha}_4.
	\end{align*}
	Using Theorem \ref{thm:sa-finite-time-bound},
	with $\epsilon_k = \epsilon/(k+K)$, we have that
	\begin{align*}
		&\mathbb{E}\left[\|V_k-V_{\pi_{\bar{\rho}}}\|_\infty^2\right]\\
		\leq \;&\alpha_1\|V_0-V_{\pi_{\bar{\rho}}}\|_\infty^2\prod_{j=0}^{k-1}(1-\alpha_2\epsilon_j)+\alpha_4A(1+2\|V_{\pi_{\bar{\rho}}}\|_\infty^2)\sum_{i=0}^{k-1}\epsilon_i^2\prod_{j=i+1}^{k-1}(1-\alpha_2\epsilon_j) \\
		\leq \;&\bar{\alpha}_1\|V_0-V_{\pi_{\bar{\rho}}}\|_\infty^2\prod_{j=0}^{k-1}(1-\bar{\alpha}_2\epsilon_j)+\bar{\alpha}_4A(1+2\|V_{\pi_{\bar{\rho}}}\|_\infty^2)\sum_{i=0}^{k-1}\epsilon_i^2\prod_{j=i+1}^{k-1}(1-\bar{\alpha}_2\epsilon_j).
	\end{align*}
	
	Now, using the same proof that leads us from Theorem \ref{thm:sa-finite-time-bound} to Corollary \ref{co:diminishing_step_size} with $\alpha_1$ to $\alpha_4$ replaced by $\bar{\alpha}_1$ to $\bar{\alpha}_4$, we have when $\xi=1$, $\epsilon=2/\bar{\alpha}_2$, and $K=\bar{\alpha}_3/\bar{\alpha}_2$ (the third case of Corollary \ref{co:diminishing_step_size}):
	\begin{align*}
		\mathbb{E}\left[\|V_k-V_{\pi_{\bar{\rho}}}\|_\infty^2\right]
		&\leq\bar{\alpha}_1\|V_0-V_{\pi_{\bar{\rho}}}\|_\infty^2\left(\frac{K}{k+K}\right)^{\epsilon\bar{\alpha}_2}+\frac{4eA\epsilon^2\bar{\alpha}_4}{\bar{\alpha}_2\epsilon-1}\frac{(1+2\|V_{\pi_{\bar{\rho}}}\|_\infty^2)}{k+K}\\
		&\leq  \left(\frac{2\bar{\alpha}_1\bar{\alpha}_3\|V_0-V_{\pi_{\bar{\rho}}}\|_\infty^2+16eA\bar{\alpha}_4(1+2\|V_{\pi_{\bar{\rho}}}\|_\infty^2)}{\bar{\alpha}_2^2}\right)\frac{1}{k+K}.
	\end{align*}
	Since
	\begin{align*}
		&\frac{2\bar{\alpha}_1\bar{\alpha}_3\|V_0-V_{\pi_{\bar{\rho}}}\|_\infty^2+16eA\bar{\alpha}_4(1+2\|V_{\pi_{\bar{\rho}}}\|_\infty^2)}{\bar{\alpha}_2^2}\\
		= \;&\frac{4}{(1-\gamma)^2}\left[\frac{96 e(A+2)\log |\mathcal{S}|\|V_0-V_{\pi_{\bar{\rho}}}\|_\infty^2}{(1-\gamma)}+\frac{256e^2 A\log |\mathcal{S}|(1+2\|V_{\pi_{\bar{\rho}}}\|_\infty^2)}{(1-\gamma)}\right]\\
		\leq \;&\frac{1024 e^2(A+2)\log |\mathcal{S}|(\|V_0-V_{\pi_{\bar{\rho}}}\|_\infty^2+2\|V_{\pi_{\bar{\rho}}}\|_\infty^2+1)}{(1-\gamma)^3},
	\end{align*}
	we have for all $k\geq 0$:
	\begin{align*}
		\mathbb{E}\left[\|V_k-V_{\pi_{\bar{\rho}}}\|_\infty^2\right]
		\leq 1024e^2(\|V_0-V_{\pi_{\bar{\rho}}}\|_\infty^2+2\|V_{\pi_{\bar{\rho}}}\|_\infty^2+1)\frac{ (A+2)\log |\mathcal{S}|}{(1-\gamma)^3}\frac{1}{k+K}.
	\end{align*}
	
	\subsection{Proof of Proposition \ref{prop:TDn}}
	\begin{enumerate}\label{ap:pf:proposition_TDn}
		\item For any $n\geq 1$, define the operator $\mathcal{T}_n$ by
		\begin{align*}
			[\mathcal{T}_n(V)](s)=\mathbb{E}_\pi\left[\sum_{i=0}^{n-1}\beta^{i}\mathcal{R}(S_i,A_i)+\beta^n V_k(S_{n})\;\middle|\;S_0=s\right].
		\end{align*}
		Note that $\mathcal{T}_n(V)$ can be written explicitly in the vector form as $\mathcal{T}_n(V)=\beta^n P_\pi^nV+\sum_{i=0}^{n-1}\beta^iP_\pi^iR_\pi$, where $P_\pi$ denotes that transition probability matrix of the Markov chain $\{S_k\}$ induced by the policy $\pi$, and  $R_\pi(s)=\sum_{a\in\mathcal{A}}\pi(a|s)\mathcal{R}(s,a)$ for all $s$. Using the operator $\mathcal{T}_n$, the TD$(n)$ algorithm can be equivalently written as
		\begin{align*}
			V_{k+1}=V_k+\epsilon_k\left(\mathcal{T}_n(V_k)-V_k+w_k\right),
		\end{align*}
		where $w_k(s)=\sum_{i=0}^{n-1}\beta^i\mathcal{R}(S_{k+i}^s,A_{k+i}^s)+\beta^n(V_k(S_{k+n}^s)-[\mathcal{T}_n(V_k)](s)$. 
		\item We first show in the following that $\|P_\pi^nV\|_\Lambda\leq \|V\|_\Lambda$ for any $V\in\mathbb{R}^{|\mathcal{S}|}$ and $n\geq 1$:
		\begin{align*}
			\|P_\pi^nV\|_\Lambda^2&=\sum_{s\in\mathcal{S}}\lambda(s)\left(\sum_{s'\in\mathcal{S}}P_\pi^n(s,s')V(s')\right)^2\\
			&\leq \sum_{s\in\mathcal{S}}\lambda(s)\sum_{s'\in\mathcal{S}}P_\pi^n(s,s')V(s')^2\tag{Jensen's Inequality}\\
			&= \sum_{s'\in\mathcal{S}}V(s')^2\sum_{s'\in\mathcal{S}}\lambda(s)P_\pi^n(s,s')\tag{Change of summation order}\\
			&= \sum_{s'\in\mathcal{S}}V(s')^2\lambda(s')\tag{$\lambda^\top P_\pi^n=\lambda^\top$}\\
			&=\|V\|_\Lambda^2.
		\end{align*}
		Now, for any $V_1,V_2\in \mathbb{R}^{|\mathcal{S}|}$, we have
		\begin{align*}
			\|\mathcal{T}_n(V_1)-\mathcal{T}_n(V_2)\|_\Lambda=\beta^n\|P_\pi^n(V_1-V_2)\|_\Lambda\leq \beta^n\|V_1-V_2\|_\Lambda,
		\end{align*}
		which implies $\mathcal{T}_n$ is a $\beta^n$-contraction with respect to $\|\cdot\|_\Lambda$.
		\item Consider the noise sequence $\{w_k\}$. By definition we have $\mathbb{E}[w_k\mid\mathcal{F}_k]=0$. Moreover, we have for any $s\in\mathcal{S}$:
		\begin{align*}
			\mathbb{E}[w_k(s)^2\mid \mathcal{F}_k]&=\text{Var}(w_k(s)\mid \mathcal{F}_k)\\
			&\leq 2\text{Var}\left(\sum_{i=0}^{n-1}\beta^{i}\mathcal{R}(S_i,A_i)\;\middle|\; \mathcal{F}_k\right)+2\beta^{2n}\text{Var}(V_k(S_n)\mid \mathcal{F}_k)\\
			&\leq 2\frac{(1-\beta^n)^2}{(1-\beta)^2}+2\beta^{2n}\sum_{s'\in\mathcal{S}}P_\pi^n(s,s')V_k(s')^2.
		\end{align*}
		It follows that 
		\begin{align*}
			\mathbb{E}[\|w_k\|_\Lambda^2\mid \mathcal{F}_k]&\leq 2\frac{(1-\beta^n)^2}{(1-\beta)^2}+2\beta^{2n}\sum_{s\in\mathcal{S}}\lambda(s)\sum_{s'\in\mathcal{S}}P_\pi^n(s,s')V_k(s')^2\\
			&= \frac{2(1-\beta^n)^2}{(1-\beta)^2}+2\beta^{2n}\|V_k\|_\Lambda^2.
		\end{align*}
	\end{enumerate}

	\subsection{Proof of Theorem \ref{thm:TDn}}\label{ap:pf:TDn}
	From Proposition \ref{prop:TDn}, we see that $\|\cdot\|_c=\|\cdot\|_e=\|\cdot\|_\Lambda$. Let $\mu=1$ and $\|\cdot\|_s=\|\cdot\|_\Lambda$. Note that $\frac{1}{2}\|x\|_\Lambda^2$ is $1$ -- smooth with respect to $\|\cdot\|_\Lambda$. Then the constants $\{\alpha_i\}_{1\leq i\leq 4}$ given in Section \ref{subsec:recurse-contract} are $\alpha_1=1$, $\alpha_2=1-\beta^n$, $\alpha_3=16(\beta^{2n}+1)$, and $\alpha_4=4$. It then follows from Corollary \ref{co:constant_step_size} that when $\epsilon_k\equiv \epsilon\leq \frac{1-\beta^n}{16(\beta^{2n}+1)}$, we have for all $k\geq 0$:
	\begin{align*}
		\mathbb{E}[\|V_k-V_\pi\|_\Lambda^2]&=\|V_0-V_\pi\|_\Lambda^2(1-(1-\beta^n)\epsilon)^k+\frac{8}{1-\beta^n}\left(\frac{(1-\beta^n)^2}{(1-\beta)^2}+2\beta^{2n}\|V_\pi\|_\Lambda^2\right)\epsilon.
	\end{align*}

	\subsection{Proof of Proposition \ref{prop:synchronous_Q}}
	\begin{enumerate}
		\item Using the definition of the Bellman's optimality operator $\mathcal{T}$, the $Q$-learning algorithm (\ref{algo:Q-learning}) can be written as
		\begin{align*}
			Q_{k+1}=Q_k+\epsilon_k(\mathcal{T}(Q_k)-Q_k+w_k),
		\end{align*}
		where $w_k(s)=\mathcal{R}(s,a)+\beta\max_{a'\in\mathcal{A}}Q_k(s'_{s,a},a')-[\mathcal{T}(Q_k)](s,a)$ for all $(s,a)$.
		\item For any $Q_1,Q_2:\mathbb{R}^{|\mathcal{S}||\mathcal{A}|}\mapsto\mathbb{R}$, we have for any state-action pairs $(s,a)$:
		\begin{align*}
			|[\mathcal{T}(Q_1)](s,a)-[\mathcal{T}(Q_2)](s,a)|&\leq \beta\mathbb{E}\left[\;\left|\max_{a'\in\mathcal{A}}Q_1(s',a')-\max_{a'\in\mathcal{A}}Q_2(s',a')\right|\;\bigg|\;s,a\right]\\
			&\leq \beta\mathbb{E}\left[\max_{a'\in\mathcal{A}}|Q_1(s',a')-Q_2(s',a')|\;\bigg|\;s,a\right]\\
			&\leq \beta\|Q_1-Q_2\|_\infty.
		\end{align*}
		Hence $\mathcal{T}$ is a $\beta$-contraction w.r.t. $\|\cdot\|_\infty$.
		\item Due to the Markov property we have $\mathbb{E}\left[w_k\mid\mathcal{F}_k\right]=0$. Moreover, since $|w_k(s,a)|\leq2\beta\|Q_k\|_\infty\leq 2(1+\|Q_k\|_\infty)$, we have $\mathbb{E}\left[\|w_k\|_\infty^2\mid\mathcal{F}_k\right]\leq 8(1+\|Q_k\|_\infty^2)$.
	\end{enumerate}

	\subsection{Proof of Theorem \ref{thm:Q-learning}}
	\begin{enumerate}
		\item Using Corollary \ref{co:log-dependence}, letting $g(x)=\frac{1}{2}\|x\|_p^2$ with $p=4\log (|\mathcal{S}||\mathcal{A}|)$ and $\mu=(1/2+1/(2\beta))^2-1$, we have in this problem
		\begin{align*}
			\alpha_1&\leq \frac{3}{2}:=\bar{\alpha}_1,&\quad
			\alpha_2&\geq \frac{1-\beta}{2}:=\bar{\alpha}_2\\
			\alpha_3&\leq  \frac{320e\log (|\mathcal{S}||\mathcal{A}|)}{1-\beta}:=\bar{\alpha}_3,&\quad
			\alpha_4&\leq \frac{128e\log (|\mathcal{S}||\mathcal{A}|)}{1-\beta}:=\bar{\alpha}_4.
		\end{align*}
		Applying Corollary \ref{co:constant_step_size} with $\epsilon\leq \bar{\alpha}_2/\bar{\alpha}_3\leq \alpha_2/\alpha_3$, we have
		\begin{align*}
			\mathbb{E}\left[\|Q_k-Q^*\|_\infty^2\right]&\leq{\alpha}_1\|Q_0-Q^*\|_\infty^2(1-\alpha_2\epsilon)^k+\alpha_4\epsilon(1+2\|Q^*\|_\infty^2)/\alpha_2\\
			&\leq\bar{\alpha}_1\|Q_0-Q^*\|_\infty^2(1-\bar{\alpha}_2\epsilon)^k+\bar{\alpha}_4\epsilon(1+2\|Q^*\|_\infty^2)/\bar{\alpha}_2\\
			&= \frac{3}{2}\|Q_0-Q^*\|_\infty^2\left[1-\frac{1}{2}\left(1-\beta\right)\epsilon\right]^k+\frac{256e  \log (|\mathcal{S}||\mathcal{A}|)(1+2\|Q^*\|_\infty^2)}{(1-\beta)^2}\epsilon.
		\end{align*}
		\item When $\epsilon_k=\epsilon/(k+K)$ with $\epsilon=\frac{2}{\bar{\alpha}_2}=\frac{4}{1-\beta}$ and $K=\frac{\bar{\alpha}_3}{\bar{\alpha}_2}=\frac{640e \log|\mathcal{S}||\mathcal{A}|}{(1-\beta)^3}$, we have by Corollary \ref{co:diminishing_step_size} (c) that
		\begin{align*}
			\mathbb{E}\left[\|Q_k-Q^*\|_\infty^2\right]&\leq \alpha_1\|Q_0-Q^*\|_\infty^2\left(\frac{K}{k+K}\right)^{\alpha_2\epsilon}+\frac{4e\epsilon^2\alpha_4(1+2\|Q^*\|_\infty^2)}{\alpha_2\epsilon-1}\frac{1}{k+K}\\
			&\leq \bar{\alpha}_1\|Q_0-Q^*\|_\infty^2\left(\frac{K}{k+K}\right)^{\bar{\alpha}_2\epsilon}+\frac{4e\epsilon^2\bar{\alpha}_4(1+2\|Q^*\|_\infty^2)}{\bar{\alpha}_2\epsilon-1}\frac{1}{k+K}\\
			&\leq 8192e^2(1+2\|Q^*\|_\infty^2+\|Q_0-Q^*\|_\infty^2)\frac{\log (|\mathcal{S}||\mathcal{A}|)}{(1-\beta)^3}\frac{1}{k+K}.
		\end{align*}
	\end{enumerate}

\end{document}